\documentclass[10pt,twocolumn]{article}

\usepackage{generic}
\usepackage{titlesec}

\usepackage{natbib}
\usepackage{algorithm}
\usepackage{algorithmic}
\usepackage{microtype}
\usepackage{graphicx}
\usepackage{subfigure}
\usepackage{booktabs} 

\usepackage{hyperref}

\usepackage{amsmath}
\usepackage{amssymb}
\usepackage{mathtools}
\usepackage{amsthm}
\usepackage{bbm}

\usepackage{tikz}
\usetikzlibrary{external}
\usetikzlibrary{arrows}
\usetikzlibrary{trees}
\usetikzlibrary{positioning}
\usetikzlibrary{shapes.geometric}

\usepackage[capitalize,noabbrev]{cleveref}

\usepackage{amsfonts}       
\usepackage{colortbl}
\usepackage{multirow}
\usepackage{eucal}
\usepackage[eulergreek]{sansmath}

\usepackage{pgfplots}
    \pgfplotsset{
        discard if not/.style 2 args={
            filter discard warning=false,
            x filter/.append code={
                \edef\tempa{\thisrow{#1}}
                \edef\tempb{#2}
                \ifx\tempa\tempb
                \else
                    
                \fi
            },
        },
    }
    \pgfplotsset{compat=1.17}

\newcommand{\R}{\mathbb{R}}
\newcommand{\Z}{\mathbb{Z}}
\newcommand{\E}{\mathbb{E}}
\newcommand{\PP}{\mathbb{P}}

\DeclareMathOperator{\argmax}{arg\, max}
\theoremstyle{plain}
\newtheorem{theorem}{Theorem}[section]
\newtheorem{proposition}[theorem]{Proposition}
\newtheorem{lemma}[theorem]{Lemma}
\newtheorem{corollary}[theorem]{Corollary}
\theoremstyle{definition}
\newtheorem{definition}[theorem]{Definition}

\theoremstyle{remark}
\newtheorem{remark}[theorem]{Remark}

\title{\textsc{LegendreTron}: Uprising Proper Multiclass Loss Learning}
\date{\vspace{-5ex}}
\author{
    Kevin H. Lam
    \thanks{School of Mathematics \& Statistics, UNSW Sydney, Australia}
    \\{\small \texttt{khflam@gmail.com}}
    \and
    Christian Walder
    \thanks{Google Research}
    \textsuperscript{,}
    \thanks{ANU College of Engineering, Computing and Cybernetics, The Australian National University, Australia}
    \\{\small \texttt{cwalder@google.com}}
    \and
    Spiridon Penev
    \footnotemark[1]
    \textsuperscript{,}
    \thanks{UNSW Data Science Hub (uDASH), UNSW Sydney, Australia}
    \\{\small \texttt{s.penev@unsw.edu.au}}
    \and
    Richard Nock
    \footnotemark[2]
    \textsuperscript{,}
    \footnotemark[3]
    \\{\small \texttt{richardnock@google.com}}
}

\begin{document}
\maketitle
\let\thefootnote\relax\footnote{This manuscript extends the ICML 2023 paper with the same name (see \citet{lam2023legendretron}).}

\begin{abstract}
Loss functions serve as the foundation of supervised learning and are often chosen prior to model development. To avoid potentially ad hoc choices of losses, statistical decision theory describes a desirable property for losses known as \emph{properness}, which asserts that Bayes' rule is optimal. Recent works have sought to \emph{learn losses} and models jointly. Existing methods do this by fitting an inverse canonical link function which monotonically maps $\R$ to $[0,1]$ to estimate probabilities for binary problems. In this paper, we extend monotonicity to maps between $\R^{C-1}$ and the projected probability simplex $\Tilde{\Delta}^{C-1}$ by using monotonicity of gradients of convex functions. We present \textsc{LegendreTron} as a novel and practical method that jointly learns \emph{proper canonical losses} and probabilities for multiclass problems. Tested on a benchmark of domains with up to 1,000 classes, our experimental results show that our method consistently outperforms the natural multiclass baseline under a $t$-test at 99\% significance on all datasets with greater than $10$ classes.
\end{abstract}

\section{Introduction}
Loss functions are a pillar of machine learning (ML). In supervised learning, a loss provides a measure of discrepancy between the underlying ground truth and a model's predictions. A learning algorithm attempts to minimise this discrepancy by adjusting the model. In other words, the loss governs how a model learns. The consequence of the bad choice of a loss is oblivious to the qualities of the learning pipeline: it means a poor model in the end. This brings forth the question: which loss is best for the problem at hand?

Statistical decision theory answers this by turning to admissible losses \citep{savage1971elicitation}; also referred to as proper losses or proper scoring rules \citep{gneiting2007proper}. Proper losses are those for which the posterior expected loss value is minimised when probability predictions coincide with the true underlying probabilities. That is, a \emph{proper} loss is one that can induce probability estimates that are \emph{admissible} or optimal. Proper losses have been extensively studied in \citet{shuford1966admissible,grunwald2004proper, reid2010binary, williamson2016multiclass}, with the latter two works extending losses to proper composite forms in binary and multiclass settings. Only a handful of proper losses, such as the square and log losses, are commonly used in ML. This is not surprising: properness is an intensional property and does not provide any candidate function. While eliciting some members is possible, extending further requires tuning or adapting the loss as part of the ML task.

There has been a recent surge of interest in doing so for supervised learning, including \citet{mei2018silvar, grabocka2019learning, streeter2019learning, liu2020unified, siahkamari2020bregman, sypherd2022tunable}. However, no connections are made to properness to formulate the losses in these works. On the other hand, several recent works have used properness to formulate losses including \citet{nock2008efficient, nock2020supervised, walder2020all,sypherd2022improper}. Notably, the works of \citet{nock2020supervised, walder2020all} have proposed algorithms to learn both the link function and linear predictor of logistic regression models by considering both functions to be unknown but learnable; thereby extending Single Index Models \citep{hardle1993sim, mei2018silvar} and algorithms to learn them \citep{kakade2011efficient}. Despite the impressive progress in these works, no references have been made to proper losses for multiclass problems.

\paragraph{Background} To approach multiclass problems in a principled manner, we generalise logistic regression as follows. For a given invertible and monotonic (see Definition \ref{definition:monotone}) link function $\psi$ that maps $[0,1]$ to $\R$ and an input-label pair $(\mathbf{x},y)$ with $\mathbf{x} \in \R^p$ and $y\in \{-1,1\}$, logistic regression learns a model of the form $\Pr(Y=1|\mathbf{x}) = \psi^{-1}(\mathbf{w}^{\top} \mathbf{x} + b)$ by fitting a coefficient vector $\mathbf{w} \in \R^p$ and an intercept $b \in \R$. A class prediction is then formed as $\hat{y} = \argmax_{y\in \{-1,1\}}\Pr(Y=y|\mathbf{x})$. The crucial element of logistic regression lies in the invertible and monotonic link function that connects probabilities to predictors. Invertibility of the link allows one to identify a unique probability estimate to associate with the predictor. Monotonicity of the link enforces an order to class predictions as elements of $\mathbf{x}$ either increase or decrease monotonically, so that the decision boundary between classes is unique. Loosely speaking, the generalisation of these ideas to multiclass problems with $C \geq 2$ classes is to form probability estimates by using a monotonic link function $\psi$ such that $\psi^{-1}(\mathbf{x}) = (p_1,p_2,\dots,p_{C-1})$ with $\sum_{k=1}^{C-1} p_k \leq 1$.

\paragraph{Motivation} In this work, our interest lies in learning proper losses for multiclass problems. Two observations highlight why this is beneficial: properness directly enforces the same ranking of classes as probabilities without building multiple models, and learned losses can provide better models for related domains. We first note that to approach a multiclass problem with $C>2$ classes, one would typically pose the problem as multiple 1-vs-rest or 1-vs-1 component problems. Each component problem consists of \emph{positive} and \emph{negative} labels where the former refers to a class of interest, and the latter refers to all other classes in 1-vs-rest or to a single other class of interest in 1-vs-1. An unfortunate consequence in the design of these reductions to binary problems is that they do not include the \textit{admissibility} constraint that probability estimates should rank classes in the same way that true probabilities do. Without loss of generality to the 1-vs-1 approach, we observe this in the following theorem.
\begin{theorem}
\label{theorem:admissibility_1_vs_rest}
Suppose we use the 1-vs-rest approach to estimate probabilities for a multiclass problem with $C>2$ classes. Then we learn models of the form 
\begin{align*}
    \Pr(\Tilde{Y}=c|\mathbf{x}) = \psi_k^{-1}(\mathbf{w}_k^{\top} \mathbf{x} + b_k)
\end{align*}
where 
$c = 
\begin{cases}
    +1 &\text{when } y = k\\
    -1 &\text{otherwise}
\end{cases}$
for $k=1,\dots,C$. Probability estimates for any class $k$ is admissible if and only if $\psi_k^{-1}(\mathbf{w}_k^{\top} \mathbf{x} + b_k) > \psi_i^{-1}(\mathbf{w}_i^{\top} \mathbf{x} + b_i)$ for all $i\neq k$.
\end{theorem}
To avoid solving $C$ 1-vs-rest problems through constrained optimisation, we desire an approach that allows us to model multiclass probabilities \textit{simultaneously}, while learning proper multiclass losses which can induce admissible probability estimates for all $C$ classes directly. It has also been shown in the work of \citet{nock2020supervised} that loss learning can provide better models for problems in domains related to the original problem where the loss was learned; compared to using uninformed losses such as cross-entropy or log loss. In general, linear models are known to be sensitive to training noise; with the notable result of \citet{long2008random} that such noise is sufficient to deteriorate any linear binary classification model to the point that it performs no better than an unbiased coin flip  on the original noise-free domain. The presence of \emph{label noise} in a dataset can be interpreted as a domain that relates to an original noise-free domain, up to some classification noise process. The ideas of loss learning and loss transfer from \citet{nock2020supervised} can then be seen as a mechanism that allows us to both overcome training noise and learn accurate models.  We illustrate that learning proper multiclass losses can be done by modelling the \emph{canonical link function} which connects probability estimates with a proper loss (see Definition \ref{definition:canonical_link} and remarks therein). In order to model canonical links flexibly, we form them as composite functions with a fixed component and a learnable component. 

\paragraph{Contributions}  Our main contributions are as follows:
\begin{itemize}
    \item We derive necessary and sufficient conditions for a composite function in $\R^{C-1}$ to be monotonic and the gradient of a twice-differentiable convex function;

    \item We derive sufficient conditions for a composite function in $\R^{C-1}$ to be monotonic and the gradient of a twice-differentiable strictly convex function;
    
    \item We present \textsc{LegendreTron} as a novel and practical way of learning proper canonical losses and probabilities concurrently in the multiclass problem setting.
\end{itemize}

\paragraph{Organisation} In Section \ref{section:related_work}, we review existing works which similarly aim to learn losses and models concurrently. In Section \ref{section:losses_multiclass}, we first describe properness and proper canonical losses. In Section \ref{section:designing_multiclass_links}, we design multiclass canonical link functions through Legendre functions and the $(u,v)$-geometric structure, and provide conditions for composite functions to be monotonic and gradients of convex functions. We then describe our method, \textsc{LegendreTron}, in detail within Section \ref{section:legendretron}. Lastly, numerical comparisons are provided in Section \ref{section:experiments} before concluding in Section \ref{section:conclusion}.

\section{Related Work}
\label{section:related_work}
\paragraph{\textsc{Tron} family of link-learning algorithms}
The notion of searching for proper losses was first established within \citet{nock2008efficient}. The \textsc{SlIsotron} algorithm was later presented in \citet{kakade2011efficient}, as the first algorithm designed to learn a model of the form $\Pr(Y=1|\mathbf{x})=u(\mathbf{w}^{\top}\mathbf{x})$ for binary problems, which involves learning the unknown link function $u:\R \to [0,1]$ assumed to be $1$-Lipschitz and non-decreasing, and the vector $\mathbf{w}\in \R^p$ used to form the linear predictor $\mathbf{w}^{\top}\mathbf{x}$. The algorithm iterates between \emph{Lipschitz isotonic regression} to estimate $u$ and gradient updates to estimate $\mathbf{w}$. A notable and practical shortcoming of \textsc{SlIsotron} is that the isotonic regression steps to update $u$ do not guarantee $u$ to map to $[0,1]$. The \textsc{BregmanTron} algorithm was later proposed in \citet{nock2020supervised}, to refine the \textsc{SlIsotron} algorithm by addressing this and providing convergence guarantees. By utilising the connection between proper losses and their canonical link functions outlined in Section \ref{section:designing_multiclass_links}, the \textsc{BregmanTron} replaced the link function $u$ with the inverse canonical link $\Tilde{\psi}^{-1}$ which guaranteed probability estimates to lie in $[0,1]$.

\paragraph{\textsc{ISGP-Linkgistic} algorithm}
The idea of using the $(u,v)$-geometric structure in combination with Legendre functions to learn canonical link functions has recently been explored in the work of \citet{walder2020all} to propose the \textsc{ISGP-Linkgistic} algorithm to learn a model of the form $\Pr(Y=1|\mathbf{x})=(u\circ v^{-1})(\mathbf{w}^{\top}\mathbf{x})$. By the squaring and integration of a Gaussian Process (GP) to yield the Integrated Squared Gaussian Process (ISGP), monotonicity and invertibility of $v^{-1}:\R \to \R$ is guaranteed. The \textsc{ISGP-Linkgistic} algorithm exploits this property by choosing a fixed squashing function $u$ separate from the \textit{a priori} ISGP distributed $v^{-1}$. Inference is performed with a stochastic EM algorithm where the $E$-step fixes the linear predictor $\mathbf{w}^{\top}\mathbf{x}$ and applies a Laplace approximation to the latent GP to compute $\mathbb{E}_{q(v^{-1}|\mathbf{w})}[\log p(y|\mathbf{x}, v^{-1})]$, and the $M$-step maximises this expectation with respect to $\mathbf{w}$. The \textsc{ISGP-Linkgistic} algorithm takes a Bayesian approach to learning proper canonical losses jointly with a probability estimator by posterior sampling of inverse canonical links.

\section{Definitions and Properties of Losses}
\label{section:losses_multiclass}
In this section, we revisit the notions of proper losses to formulate proper canonical losses in the multiclass setting. We follow the definitions and notations of \citet{williamson2016multiclass} and describe key properties therein, for our discussion of composite multiclass losses. 

Let $C\geq 2$ be the total number of classes. Our setting is multiclass probability estimation. Denote the $(C-1)$-dimensional probability simplex as
\begin{align*}
    \Delta^{C-1} = \left\{p \in \R^{C}_+: \sum_{i=1}^{C} p_i = 1 \right\},
\end{align*}
and its relative interior as
\begin{align*}
    \mathrm{ri}(\Delta^{C-1}) = \left\{p \in \R^{C}_+: \sum_{i=1}^{C} p_i = 1, p_i \in (0,1), \forall i \right\}.
\end{align*}

Suppose we have a dataset $\mathcal{D}$ of $N$ pairs $\{(\mathbf{x}_n, y_n)\}_{n=1}^N$ where each $\mathbf{x}_n \in \mathcal{X}=\R^p$ and $y_n \in \mathcal{Y}= \{1,\dots,C\}$ denotes an input and a single label respectively. We aim to learn a function $h:\mathcal{X} \to \Delta^{C-1}$ such that $\hat{y}_n \in \argmax_{c\in \{1,\dots,C\}} \PP(y_n=c|\mathbf{x}_n)$ closely matches $y_n$. 

Consider the label as a random variable $Y \sim \text{Categorical}(p)$ with prior class probabilities $p\in \Delta^{C-1}$. We denote $q\in \Delta^{C-1}$ as the estimated probabilities in the following definitions. To assess the quality of probability estimates, a loss function can be defined generally as 
\begin{align*}
    \ell: \Delta^{C-1} \to \R^C_+,\quad \ell(q) = (\ell_1(q), \dots, \ell_C(q))^{\top}
\end{align*}
where $\ell_i$ is the \emph{partial loss} for predicting $q$ when $y=i$. For a given label $y$, we can return to \emph{scalar}-valued losses by referring to the $y$-th partial loss $\ell_y$. 

\begin{definition}[conditional Bayes Risk]
\label{definition:bayes_risk}
The conditional risk associated with $\ell$ is defined as $L(p,q) =\E_{Y \sim \text{Categorical}(p)}[\ell_{Y}(q)]$ for all $p,q \in \Delta^{C-1}$. The best achievable conditional risk associated with a loss is termed the \emph{conditional Bayes risk} and is defined as
\begin{align*}
    \underline{L}: &\Delta^{C-1} \to \R_+,\\
    \underline{L}(p) = \inf_{q \in \Delta^{C-1}} L(p,q) &= \inf_{q \in \Delta^{C-1}} \E_{Y \sim \mathrm{Categorical}(p)}[\ell_{Y}(q)].
\end{align*}
\end{definition}
It is well known that $\underline{L}$ is concave.
 
\begin{definition}[Proper Losses]
\label{definition:proper_losses}
A loss $\ell$ is \emph{proper} if and only if $L$ is minimized when $q=p$. In other words, $\underline{L}(p)=L(p,p) \leq L(p,q)$ for all $p,q\in \Delta^{C-1}$.  Losses where the inequality is strict when $p\neq q$,  are termed \emph{strictly proper}.
\end{definition}

\paragraph{Remark} Properness is an essential property of losses, as optimising a model with respect to a proper loss guides the model's probability estimates towards true posterior class probabilities. Examples of proper losses include the $0-1$, square, log, and Matsushita losses \citep{matusita1956decision}.

To draw the connection between a proper loss and its conditional Bayes risk, we require definitions of subgradients and Bregman divergences. Subgradients are a generalisation of gradients and are particularly useful when analysing convex functions that may not be differentiable.
\paragraph{Subgradients} For a convex set $S\subseteq \R^n$, the subdifferential of a convex function $f: S \rightarrow(-\infty,+\infty]$ at $\mathbf{x} \in S$ is defined as
\begin{align*}
    \partial f(\mathbf{x})=\left\{\phi \in \mathbb{R}^{n}:\langle\phi, \mathbf{y}-\mathbf{x}\rangle \leq f(\mathbf{y})-f(\mathbf{x}), \forall \mathbf{y} \in \mathbb{R}^{n}\right\}
\end{align*}
where an element $\phi \in \partial f(\mathbf{x})$ is called a \emph{subgradient} of $f$ at $\mathbf{x}$. By convention, we define $\partial f(\mathbf{x})=\emptyset$ for all $\mathbf{x} \notin S$. Moreover, $f$ is strictly convex if and only if $\partial f(\mathbf{x})=\left\{\phi \in \mathbb{R}^{n}:\langle\phi, \mathbf{y}-\mathbf{x}\rangle < f(\mathbf{y})-f(\mathbf{x}), \forall \mathbf{y} \in \mathbb{R}^{n}\right\}$.

\paragraph{Bregman divergence} For a convex set $S\subseteq \R^n$, and a continuously-differentiable and strictly convex function $f: S \rightarrow (-\infty,+\infty]$, the Bregman divergence with generator $f$ is defined for all $\mathbf{x},\mathbf{y} \in S$ as 
\begin{align*}
    D_f(\mathbf{x}, \mathbf{y}) = f(\mathbf{x}) - f(\mathbf{y}) - \langle \nabla f(\mathbf{y}),\mathbf{x}-\mathbf{y} \rangle.
\end{align*}

The following result is a rewritten characterisation of proper losses through their ``Bregman representation'', and explicates the connection between a proper loss and its conditional Bayes risk.
\begin{proposition}[{\citep[Proposition 7]{williamson2016multiclass}}]
\label{proposition:proper_loss_bregman}
Let $\ell:\Delta^{C-1} \to \R_+^C$ be a loss. $\ell$ is a (strictly) proper loss if and only if there exists a (strictly) convex function $f:\Delta^{C-1} \to \R$ such that for all $q\in \Delta^{C-1}$, there exists a subgradient $\phi \in \partial f(q)$ such that 
\begin{align*}
    L(p,q) = -(p-q)^{\top} \phi - f(q) \text{ for all $p \in \Delta^{C-1}$}.
\end{align*}
Moreover, if $\underline{L}$ is differentiable on $\text{ri}(\Delta^{C-1})$ then 
\begin{align*}
    L(p,q) = (p-q)^{\top} \ell(q) + \underline{L}(q)
\end{align*}
where $\ell$ is the unique proper loss associated with $\underline{L}$ with the property $\nabla \underline{L}(p) = \ell(p), \forall p \in \text{ri}(\Delta^{C-1})$.
\end{proposition}
\begin{remark}
We note that $L(p,q)-\underline{L}(p)$ is a Bregman divergence if and only if $\ell$ is strictly proper due to the requirement of strict convexity of $-\underline{L}$.
\end{remark}

In this work, we seek to learn strictly proper losses $\ell$ by exploiting the connection $\nabla\underline{L}(p) = \ell(p)$ described in Proposition \ref{proposition:proper_loss_bregman}. In Section \ref{section:designing_multiclass_links}, we extend this connection between probabilities and predictors in $\R^{C-1}$ through \emph{canonical link functions}, and describe in detail how strictly proper losses can be learned through this extended connection.

\section{Designing Multiclass Canonical Links}
\label{section:designing_multiclass_links}
In this section, we provide definitions of canonical link functions, Legendre functions and the $(u,v)$-geometric structure. The latter two structures are essential for the design and learning of canonical link functions. We show that designing a canonical link amounts to designing a composite function that is the gradient of a twice-differentiable and convex function. To this end, we present our key theoretical contributions: conditions for composite functions to be gradients of convex functions.

\paragraph{Composite Form} It is often desirable to link predictors with their probability estimates through an invertible link function $\psi: \Delta^{C-1} \to \R^C$. This allows one to uniquely identify probabilities while working with general predictors. It also allows one to define loss functions more generally as $\ell_{\psi} = \ell \circ \psi^{-1}$ which are referred to as \emph{proper composite losses} when $\ell$ is proper. \citet[Proposition 13]{williamson2016multiclass} shows that a proper composite loss $\ell_{\psi}$ is uniquely represented by $\ell$ and $\psi$ when $\ell_{\psi}$ is continuous and invertible.

\paragraph{Proper Canonical Form} As elements of $\Delta^{C-1}$ are uniquely determined by the first $C-1$ components, the above properties can be more naturally described by the \emph{projected} probability simplex:
\begin{align*}
    \Tilde{\Delta}^{C-1} = \left\{\Tilde{p} \in \R^{C-1}_+: \sum_{i=1}^{C-1} \Tilde{p}_i \leq 1\right\}.
\end{align*}
Define the projection map 
\begin{align*}
    \Pi&: \Delta^{C-1} \to \Tilde{\Delta}^{C-1},\\
    \Pi(p) = (p_1, &\dots, p_{C-1}) \text{ for all $p \in \Delta^{C-1}$},
\end{align*}
and its inverse 
\begin{align*}
    \Pi^{-1}&: \Tilde{\Delta}^{C-1} \to \Delta^{C-1},\\
    \Pi^{-1}(\Tilde{p}) = \Bigg(\Tilde{p}_1, &\dots, \Tilde{p}_{C-1}, 1 - \sum_{i=1}^{C-1} \Tilde{p}_i\Bigg) \text{ for all $\Tilde{p} \in \Tilde{\Delta}^{C-1}$}.
\end{align*}
\begin{definition}
\label{definition:canonical_link}
The projected conditional Bayes risk is defined as $\Tilde{\underline{L}}=\underline{L} \circ \Pi^{-1}$. Suppose  $\Tilde{\underline{L}}$ is differentiable. Then the \emph{canonical link function} is defined as
\begin{align*}
    \Tilde{\psi}: \Tilde{\Delta}^{C-1} \to \R^{C-1}, \quad \Tilde{\psi}(\Tilde{p}) = -\nabla\Tilde{\underline{L}}(\Tilde{p}).
\end{align*}
\end{definition}
Equipping a proper loss $\ell$ with its corresponding canonical link $\Tilde{\psi}$ yields the function $\ell \circ \Pi^{-1} \circ \Tilde{\psi}^{-1}$ with its components being convex with respect to the input domain (see Appendix \ref{appendix:convexity_proper_canonical_losses}). We refer to such losses as \emph{proper canonical losses} to distinguish them from proper composite losses. The connection between a differentiable conditional Bayes risk, a proper loss, and a canonical link, shown by Proposition \ref{proposition:proper_loss_bregman} and Definition \ref{definition:canonical_link}, is explicated within Appendix \ref{appendix:convexity_proper_canonical_losses}. The unique coupling of a proper loss, canonical link and conditional Bayes risk illustrates that one can learn proper canonical losses by modelling either of the latter two functions.

\paragraph{Properties of Legendre functions} Let $f:\R^{C-1} \to \R$ be continuously differentiable and strictly convex. We refer to $f$ as a \emph{Legendre function}. The \emph{Legendre-Fenchel conjugate} of $f$, denoted by $f^*$, is defined as 
\begin{align*}
    f^*&:S\to \R,\\
    f^*(\mathbf{x}^*) = \langle (\nabla f)^{-1}(\mathbf{x}^*)&, \mathbf{x}^*\rangle - f\bigl((\nabla f)^{-1}(\mathbf{x}^*)\bigr).
\end{align*}
where $S = \{\nabla f(\mathbf{x}):\mathbf{x} \in \R^{C-1}\}$, and $f$ is Legendre if and only if $f^*$ is Legendre. \citet[Theorem 26.5]{rockafellar1970convex} shows that when the latter holds, $(f^{*})^*=f$, and $\nabla f$ is continuous and invertible with $\nabla f^* = (\nabla f)^{-1}$. Moreover, if $f$ is twice-differentiable with positive definite Hessian everywhere, then the inverse function theorem yields that $f^*$ is twice-differentiable since $\nabla^2 f^*(\nabla f(\mathbf{x})) = (\nabla^2 f(\mathbf{x}))^{-1}$.

\paragraph{$(u,v)$-geometric structure} \citet{amari2016information,nock2016conformal, walder2020all} state that a general dually flat structure on $\R^{C-1}$ can be defined in terms of an arbitrary strictly convex function $\xi$. Let $u$ and $v$ be differentiable invertible functions. The pair $(u, v)$ give a dually flat structure on $\R^{C-1}$ if and only if $\nabla \xi = u \circ v^{-1}$. We consider the $(u,v)$-geometric structure of the Bregman divergence $D_{(-\Tilde{\underline{L}})^{*}}$ which gives $\Tilde{\psi}^{-1} = u \circ v^{-1}$.

\paragraph{Designing links} In this work, we focus on the case when $-\Tilde{\underline{L}}$ is twice-differentiable. Note that $-\Tilde{\underline{L}}$ is convex since $\Pi^{-1}$ is affine and $-\underline{L}$ is convex.  Properties of Legendre functions allow us to move from $-\Tilde{\underline{L}}$ to its Legendre-Fenchel conjugate $(-\Tilde{\underline{L}})^{*}$, and similarly allow us to move from the canonical link $\Tilde{\psi}$ to its inverse $\Tilde{\psi}^{-1}$. The $(u,v)$-geometric structure then allows us to flexibly learn $\Tilde{\psi}^{-1}$ by splitting it into a \emph{learnable} component $v^{-1}$ and a fixed component $u$. Fixing $u$ to be a suitable \emph{squashing} function ensures that $\Tilde{\psi}^{-1}$ maps to $\Tilde{\Delta}^{C-1}$; thereby allowing us to uniquely identify multiclass probabilities associated with predictors from $\R^{C-1}$. On the other hand, $v^{-1}$ can be parameterised by an \emph{invertible neural network} which allows $\Tilde{\psi}^{-1}$ to adapt to the multiclass problem at hand. Legendre functions and the $(u,v)$-geometric structure together yield a more natural and practical design of the canonical link through its inverse since it is often much easier to map inputs from an unbounded space such as $\R^{C-1}$, to a bounded space such as $\Tilde{\Delta}^{C-1}$. Figures \ref{figure:overall_link_design} and \ref{figure:canonical_link_uv_structure} illustrates how the inverse of the canonical link is modelled using the $(u,v)$-geometric structure. Loosely speaking, $v^{-1}$ allows one to find better logit representations before they are squashed to probabilities.

Under the $(u,v)$-geometric structure, if one can prove that $u \circ v^{-1}$ maps to $\Tilde{\Delta}^{C-1}$ and is the gradient of a Legendre function $f$, then one can set $(-\Tilde{\underline{L}})^{*}=f$ and $\nabla (-\Tilde{\underline{L}})^{*}=u \circ v^{-1}$ as its corresponding inverse canonical link function by using properties of Legendre functions. This requires showing $u \circ v^{-1}$ is the gradient of a twice-differentiable and strictly convex function. In the following two theorems, we provide conditions where this assertion holds for general composite functions. We defer the background, supporting theorems and proofs of the following results to Sections \ref{appendix:convex_analysis}, \ref{appendix:composition_iff_psd} and \ref{appendix:maps_with_pd_jacobians_closed_under_composition} within the Appendices.

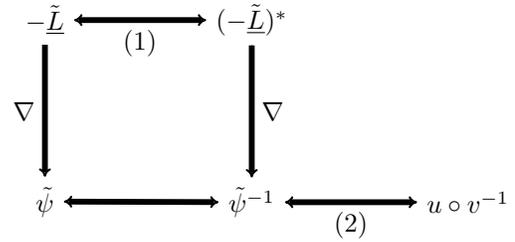
\begin{figure}
  \centering
  \begin{tikzpicture}
      \node[draw=none, thick] (negative_bayes_risk) {$-\Tilde{\underline{L}}$};
      \node[right=5em of negative_bayes_risk, draw=none, thick] (negative_bayes_risk_conjugate) {$(-\Tilde{\underline{L}})^{*}$};
      \node[below=5em of negative_bayes_risk, draw=none, thick] (canonical_link) {$\Tilde{\psi}$};
      \node[below=5em of negative_bayes_risk_conjugate, draw=none, thick] (inverse_canonical_link) {$\Tilde{\psi}^{-1}$};
      \node[right=5em of inverse_canonical_link, draw=none, thick] (composite_link) {$u \circ v^{-1}$};
      
      \draw[implies-implies,double=black, thick] (negative_bayes_risk) -- node [midway, draw=none, below] {(1)} (negative_bayes_risk_conjugate);
      \draw[implies-implies,double=black, thick] (canonical_link) -- (inverse_canonical_link);
      \draw[-implies,double=black, thick] (negative_bayes_risk) -- node [midway, draw=none, left] {$\nabla$} (canonical_link);
      \draw[-implies,double=black, thick] (negative_bayes_risk_conjugate) -- node [midway, draw=none, right] {$\nabla$}(inverse_canonical_link);
      \draw[implies-implies,double=black, thick] (inverse_canonical_link) -- node [midway, draw=none, below] {(2)} (composite_link);
  \end{tikzpicture}
\caption{Inverse canonical links as composite functions by moving from $-\Tilde{\underline{L}}$ to $(-\Tilde{\underline{L}})^{*}$ using Legendre functions in (1) and decomposing $\Tilde{\psi}^{-1}$ using the $(u,v)$-geometric structure in (2).}
\label{figure:overall_link_design}
\end{figure}

\begin{figure}
  \centering
  \includegraphics{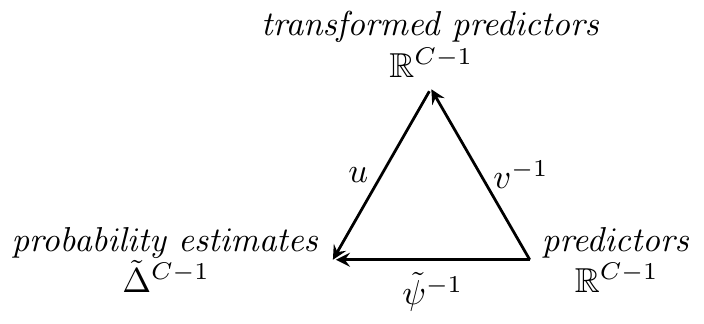}
  \caption{Relationship between predictors and probability estimates through the inverse of the canonical link function under the $(u,v)$-geometric structure.}
  \label{figure:canonical_link_uv_structure}
\end{figure}

\begin{theorem}
\label{theorem:composition_gradient_iff_psd}
Let $f:\R^{C-1} \to \R^{C-1}$ and $g:\R^{C-1} \to \R^{C-1}$ be differentiable. Then the following conditions are equivalent:
\begin{enumerate}
    \item $f \circ g = \nabla F$ where $F$ is a twice-differentiable convex function.

    \item The Jacobian $J_{f \circ g}(\mathbf{x})$ is symmetric for all $\mathbf{x}\in \R^{C-1}$.

    \item $J_{f \circ g}(\mathbf{x})$ is positive semi-definite for all $\mathbf{x}\in \R^{C-1}$.
    
    \item $f \circ g$ is monotone.
\end{enumerate}
\end{theorem}

\paragraph{Proof sketch of Theorem \ref{theorem:composition_gradient_iff_psd}} To claim that a function $f : \R^{C-1} \to \R^{C-1}$ is the gradient of a convex function $g: \R^{C-1} \to \R$, requires $f$ to satisfy \emph{maximal cyclical monotonicity}. This is a more abstract notion of monotonicity within domains in higher dimensions, and encompasses two notions of monotonicity, namely \emph{maximal monotonicity} and \emph{cyclical monotonicity}. It turns out that it is sufficient to consider monotonicity as maximal monotonicity is automatically guaranteed as our domain is $\R^{C-1}$, and $f \circ g$ is differentiable and therefore continuous. 

Theorem \ref{theorem:composition_gradient_iff_psd} characterises when a composite function is the gradient of a convex function. It also serves as a convenient and practical criteria to aid model design through a check of positive semi-definiteness for the Jacobian $J_{f \circ g}$. The implications of Theorem \ref{theorem:composition_gradient_iff_psd} are profound as it allows us to derive the following sufficient conditions under which the composition of gradients of convex functions is the gradient of a Legendre function.
\begin{theorem}
\label{theorem:maps_with_pd_jacobians_closed_under_composition}
Let $f:\R^{C-1} \to S$ and $g:\R^{C-1} \to \R^{C-1}$ be differentiable where $S \subseteq \R^{C-1}$, and $J_f(\mathbf{x})$ and $J_g(\mathbf{x})$ are symmetric and positive definite for all $\mathbf{x} \in \R^{C-1}$. Then $f \circ g$ is the gradient of a twice-differentiable Legendre function.
\end{theorem}

\paragraph{Proof sketch of Theorem \ref{theorem:maps_with_pd_jacobians_closed_under_composition}} Theorem \ref{theorem:composition_gradient_iff_psd} tells us it is sufficient to check for positive semi-definiteness of a composite function's Jacobian. Our proof involves a check that all eigenvalues of the Jacobian are positive. This asserts that the composite function is the gradient of a twice-differentiable and strictly convex function.

To use the $(u,v)$-geometric structure from Section \ref{section:losses_multiclass} with Theorem \ref{theorem:maps_with_pd_jacobians_closed_under_composition}, we can set $f=u$ and $g=v^{-1}$ within Theorem \ref{theorem:maps_with_pd_jacobians_closed_under_composition}. This presents an additional requirement that the functions $f$ and $g$ are also invertible. In Section \ref{section:legendretron}, we show how these requirements can be met with our proposed algorithm, \textsc{LegendreTron}.

\section{Learning Proper Canonical Multiclass Losses: \textsc{LegendreTron}}
\label{section:legendretron}
In this section, we present \textsc{LegendreTron}, our main algorithmic contribution for learning proper canonical losses for multiclass probability estimation. With the theory of Legendre functions, $(u,v)$-geometric structure and Theorem \ref{theorem:maps_with_pd_jacobians_closed_under_composition} in hand to support our approach, we now present \textsc{LegendreTron} in detail, as an extension of generalised linear models and Single Index Models for multinomial logistic regression. We note that the conventional formulation of multinomial logistic regression is not a generalised linear model and defer a principled reformulation to Appendix \ref{appendix:redefining_multinomial_logistic}. 

\paragraph{Model} Given a dataset $\mathcal{D}=\{(\mathbf{x}_n, y_n)\}_{n=1}^N$, we have the classification model
\begin{align*}
    y_n|\mathbf{x}_n &\sim \text{Categorical}\bigl(\hat{p}(\mathbf{z}_n)\bigr) \text{ where $\mathbf{z}_n = \mathbf{W}\mathbf{x}_n + \mathbf{b}$}
\end{align*}
where $\mathbf{W}\in \R^{(C-1)\times p}$, $\mathbf{b}\in \R^{C-1}$ and $\hat{p}(\mathbf{z}_n) = (u \circ v^{-1})(\mathbf{z}_n)$ with $u$ chosen as a squashing function that maps to $\Tilde{\Delta}^{C-1}$ and $v^{-1} = \nabla g$ for a twice-differentiable and strictly convex function $g$. We leave the specification of a suitable squashing function $u$ as a modelling choice and provide a natural choice at the end of this section.

For any $B \in \Z_+$, let $g_1, g_2, \dots g_B$ be fully input convex neural networks (FICNN) investigated in \citet{amos2017icnn}. We set $v^{-1}=\nabla g = (\nabla g_1) \circ (\nabla g_2) \circ \dots \circ (\nabla g_B)$. For each $g_i$, we use the same architecture as \citet{huang2021cpflows} which is described as
\begin{align*}
    \mathbf{z}_{i,1} &= l_{i,1}^+(\mathbf{x})\\
    \mathbf{z}_{i,k} &= l_{i,k}(\mathbf{x}) + l_{i,k}^+(s(\mathbf{z}_{i,k-1})) \text{ for $k=2,\dots,M+1$},\\
    h_i(\mathbf{x}) &= s(\mathbf{z}_{i,M+1}),\\
    g_i(\mathbf{x}) &= s(w_{i,0}) h_i(\mathbf{x}) + s(w_{i,1})\frac{\|\mathbf{x}\|^2}{2}
\end{align*}
where we denote $l_{i,k}^+$ as a linear layer with \emph{positive} weights, $l_{i,k}$ as a linear layer with unconstrained weights, $w_{i,0},w_{i,1} \in \R$ are unconstrained parameters and $s(x) =\log (1 + e^x)$ is the softplus function with $s(\mathbf{x})$ denoting the softplus function applied elementwise on $\mathbf{x}$. In particular, $l_{i,M+1}$ and $l_{i,M+1}^+$ are linear layers that map to $\R$ while for each $k=1,\dots M$, $l_{i,k}$ and $l_{i,k}^+$ are hidden layers that map to $\R^H$ for a chosen dimension size $H \in \Z_+$. With this setup, each $g_i$ is strongly convex (and therefore strictly convex) with an invertible gradient and positive definite Hessian for all $\mathbf{x}\in \R^{C-1}$ due to the quadratic term within each $g_i$.

We now show that, when equipped with a suitable squashing function $u$, any function learned by \textsc{LegendreTron} is a valid inverse canonical link function. We turn to a modified version of the \textit{LogSumExp} function previously studied in \citet{nielsen2018monte} and describe its main properties within the following theorem.

\begin{theorem}
\label{theorem:legendretron_is_legendre}
Let $f(\mathbf{x})= \log\left(1+ \sum_{k=1}^{C-1} \exp(x_k)\right)$. The key properties of $f$ are:
\begin{itemize}
    \item $f$ is strictly convex with invertible gradient
    \begin{align*}
    &u:\R^{C-1} \to \Tilde{\Delta}^{C-1},\\
    u(\mathbf{x}) &= \left(\frac{\exp(x_i)}{1+ \sum_{k=1}^{C-1} \exp(x_k)}\right)_{1\leq i \leq C-1}.
    \end{align*}
    \item the Hessian of $f$, given by $J_{u}(\mathbf{x})$, is positive definite for all $\mathbf{x} \in \R^{C-1}$.
\end{itemize}

We refer to $f$ and $u$ as $\text{LogSumExp}^+$ and $\text{softmax}^+$ respectively. Let $v^{-1}:\R^{C-1} \to \R^{C-1}$ be defined as
\begin{align*}
    v^{-1} = (\nabla g_1) \circ (\nabla g_2) \circ \dots \circ (\nabla g_B)
\end{align*}
where $g_1, g_2, \dots g_B$ are FICNNs. Then any function $u\circ v^{-1}$ learned by \textsc{LegendreTron} is the gradient of a twice-differentiable Legendre function and is therefore, the inverse of a canonical link function.
\end{theorem}

With this specification, any function $u \circ v^{-1}$ learned via \textsc{LegendreTron} is the gradient of a twice-differentiable Legendre function which can serve as an inverse canonical link function. Moreover, we can deduce that any inverse canonical link or gradient of a twice-differentiable Legendre function, can be approximated by the architecture of $u \circ v^{-1}$ defined in Theorem \ref{theorem:legendretron_is_legendre}. 
\begin{corollary}
\label{corollary:uv_approximates_arbitrary_legendre_gradients}
    Let $u:\R^{C-1} \to \Tilde{\Delta}^{C-1}$ and $v^{-1}:\R^{C-1} \to \R^{C-1}$ be defined as in Theorem \ref{theorem:legendretron_is_legendre} with $v^{-1}$ parameterised by $\boldsymbol{\theta}$. Define $\mathcal{C}(\Omega)$ as the set of twice-differentiable convex functions with positive definite Hessian everywhere for a compact set $\Omega \subset \R^{C-1}$, and $\mathcal{F} = \{f: f:\R^{C-1} \to \Tilde{\Delta}^{C-1}, \exists \boldsymbol{\theta} \text{ such that } u \circ v^{-1}=f\}$. Then $\mathcal{F}$ is dense in $\mathcal{C}(\Omega)$.
\end{corollary}

Algorithm \ref{algorithm:legendretron} describes \textsc{LegendreTron} in detail. We conclude this section with remarks on our model design, and connections between \textsc{LegendreTron}, multinomial logistic regression and the log loss.

\begin{remark}
The basis of our design of \textsc{LegendreTron} comes from requiring $J_{u \circ v^{-1}}(\mathbf{x})$ to be positive semi-definite and with a determinant that satisfies
\begin{align*}
    |J_{u \circ v^{-1}}(\mathbf{x})| = |J_{u}(v^{-1}(\mathbf{x}))||J_{v^{-1}}(\mathbf{x})| >0 \text{ for all $\mathbf{x} \in \R^{C-1}$}.
\end{align*}
A direct way to guarantee this is by setting $u$ and $v^{-1}$ to be twice-differentiable functions with positive definite Hessians. To the best of our knowledge, only the CP-Flow architecture \citep{huang2021cpflows} satisfies this property among invertible networks in normalising flows literature; making it our choice for $v^{-1}$. While it is possible to set other functions as $u$, it is generally difficult to elicit \emph{invertible} functions that squash inputs to $\Tilde{\Delta}^{C-1}$ aside from variants of softmax. We have set $u=\text{softmax}^+$ as it simplifies to the well-known sigmoid when $C=2$, which was used as the analogous squashing function in ISGP \citep{walder2020all}. 
\end{remark}

\begin{remark}
As $\text{LogSumExp}^+$ is twice-differentiable and Legendre, its gradient $\text{softmax}^+$ is a valid inverse canonical link function since it maps to $\Tilde{\Delta}^{C-1}$. However, we note that setting $\Tilde{\psi}^{-1}=\text{softmax}^+$ results in learning only the parameters $\mathbf{W}$ and $\mathbf{b}$ which is equivalent to formulating multinomial logistic regression as a generalised linear model with link $\Tilde{\psi}(\Tilde{p}) = \left(\log\left(\frac{\Tilde{p}_i}{1 - \sum_{k=1}^{C-1}\Tilde{p}_k }\right)\right)_{1\leq i \leq C-1}$ with corresponding proper loss $\ell = -((\Tilde{\psi} \circ \Pi) \cdot J_{\Pi})$ which can be shown to be the log loss (or \emph{cross-entropy}). That is, \textsc{LegendreTron} with $v^{-1}$ as the identity map is equivalent to multinomial logistic regression or logistic regression when $C=2$. In this case, Algorithm \ref{algorithm:legendretron} would only optimise parameters of a linear model without loss learning; by using the log loss. Comparisons between \textsc{LegendreTron} against multinomial logistic regression or logistic regression illustrate performance differences between learning a proper loss for the dataset and optimising with respect to log loss (cross-entropy). These results are provided in Tables \ref{table:test_auc} and \ref{table:libsvm_results}, and Figure \ref{figure:classification_accuracy}.
\end{remark}

\begin{algorithm}[tb]
   \caption{\textsc{LegendreTron}}
   \label{algorithm:legendretron}
\begin{algorithmic}
   \STATE {\bfseries Input:} sample $\mathcal{S}\subset\mathcal{D}$, number of iterations $T$, number of FICNNs $B$, hidden layer dimension size $H$, number of layers $M$, squashing function $u$.
   \STATE Initialise $\mathbf{W}$ and $\mathbf{b}$.
   \STATE Initialise $g_1, g_2, \dots g_B$ each with $M$ layers of dimension size $H$, and denote their joint set of parameters $\boldsymbol{\theta}$.
   \FOR{$i=1$ {\bfseries to} $T$}
   \STATE Set $v^{-1} = (\nabla g_1) \circ (\nabla g_2) \circ \dots \circ (\nabla g_B)$.
   \FOR{each $(\mathbf{x}_n, y_n) \in \mathcal{S}$}
   \STATE Compute $\mathbf{z}_n = \mathbf{W}\mathbf{x}_n + \mathbf{b}$.
   \STATE Compute $\hat{p}(\mathbf{z}_n)=(u \circ v^{-1})(\mathbf{z}_n)$.
   \ENDFOR
   \STATE Compute $\mathbb{E}_{\mathcal{S}}[\mathcal{L}(\hat{p}(\mathbf{z}), y)]$ by Monte Carlo where $\mathcal{L}$ is the log-likelihood of the Categorical distribution.
   \STATE Update $\mathbf{W}$, $\mathbf{b}$ and $\boldsymbol{\theta}$ by backpropagation.
   \ENDFOR
   \STATE {\bfseries Output:} $\mathbf{W}$, $\mathbf{b}$ and $g_1, g_2, \dots g_B$.
\end{algorithmic}
\end{algorithm}

\section{Experiments}
\label{section:experiments}
In this section, we provide numerical comparisons between \textsc{LegendreTron}, multinomial logistic regression and other existing methods that also aim to jointly learn models and proper canonical losses. For our experiments, we set $\text{softmax}^+$ as the squashing function $u$ for both \textsc{LegendreTron} and multinomial logistic regression. For a practical and numerically stable implementation, we also map probability estimates to the log scale by deriving an alternate Log-Sum-Exp trick for $\text{softmax}^+$. We defer the full experimental details to Appendix \ref{appendix:experimental_details}. 

All experiments were performed using \texttt{PyTorch} \citep{pytorch2019} and took roughly one CPU month to complete\footnote{The total run time for our experiments is favourable relative to the reported two CPU months for the \textsc{ISGP-Linkgistic} algorithm from \citet{walder2020all}.}. CPU run times for the aloi dataset, which had the largest number of classes ($1,000$), were respectively 4 hours and 0.75 hours for \textsc{LegendreTron} and multinomial logistic regression. We note that the difference in run times for this experiment are in part due to the larger number of epochs ($360$), larger number of blocks $B$, autograd and backpropagation operations to update $v^{-1}$ for a much larger number of classes. Average GPU run times on a P100 for MNIST experiments in Table \ref{table:vgg_acc}, were 2.32 and 2.12 hours for \textsc{VGGTron} and VGG respectively. These run times demonstrate the relative efficiency and applicability of loss learning for most datasets.

\begin{table}
  \caption{Test AUC for generalised linear models with various link methods (ordering in decreasing average). See text for details.}
  \label{table:test_auc}
  \centering
  \begin{tabular}{lll}
    \toprule
         & MNIST & FMNIST \\
    \midrule
    \rowcolor[gray]{0.7}
    \textsc{LegendreTron} & 99.9\%  & 99.2\% \\
    ISGP-Linkgistic & 99.9\%  & 99.2\% \\
    GP-Linkgistic & 99.9\% & 99.1\% \\
    Logistic regression & 99.9\% & 98.5\%  \\
    GLMTron & 99.6\% & 98.1\% \\
    \textsc{BregmanTron} & 99.7\% & 97.9\%  \\
    \textsc{BregmanTron}$_{\texttt{label}}$ & 99.6\% & 97.7\%  \\
    \textsc{BregmanTron}$_{\texttt{approx}}$ & 99.3\% & 94.6\%  \\
    \textsc{SlIsotron} & 94.6\% & 90.7\%  \\
    \bottomrule
  \end{tabular}
\end{table}

\paragraph{MNIST Binary Problems} Binary problems are a special case of our setting where $C=2$, so \textsc{LegendreTron} is readily applicable. In Table \ref{table:test_auc}, we compared \textsc{LegendreTron} against \textsc{ISGP-Linkgistic} \citep{walder2020all} and \textsc{BregmanTron} \citep{nock2020supervised}, as both algorithms also aim to learn proper canonical losses for binary problems. We also compared with other baselines in these two works including the \textsc{SlIsotron} algorithm from \citet{kakade2011efficient}. Experiment details can be found in Section 6 of \citet{nock2020supervised}. Our model successfully matches the (binary specific) \textsc{ISGP-Linkgistic} baseline, which was the strongest algorithm in test AUC performance from the experiments of \citet{walder2020all}.

\paragraph{MNIST Multiclass Problems using Linear Models} For the three MNIST-like datasets \citep{lecun2010mnist,xiao2017fashionmnist,clanuwat2018kushijimnist}, we compared \textsc{LegendreTron} against multinomial logistic regression and \textsc{ISGP-Linkgistic}, since the latter is the strongest algorithm in ten-class classification test accuracy performance based on the experiments within \citet{walder2020all}. \textsc{ISGP-Linkgistic} approaches the multiclass problem by learning proper canonical losses for the 10 component 1-vs-rest problems. Our experimental results in Figure \ref{figure:classification_accuracy} show that \textsc{LegendreTron} and multinomial logistic regression outperform the \textsc{ISGP-Linkgistic} baseline on all three datasets. These results illustrate our conjecture that properness with respect to losses and models in component problems in a multiclass setting, does not imply optimality of class predictions or probability estimates. By respecting the true problem structure, proper multiclass losses allow the model to learn probability estimates that are able to better distinguish between all the classes at hand. Our results also show that \textsc{LegendreTron} either matches or outperforms multinomial logistic regression on all three datasets. This is most notable on the Kuzushiji-MNIST dataset where \textsc{LegendreTron} outperforms multinomial logistic regression by a reasonable margin.

\paragraph{MNIST Multiclass Problems using Nonlinear Models}
We note that the architecture of $u \circ v^{-1}$ does not restrict us to linear models. In Table \ref{table:vgg_acc}, we provided experimental results of how loss learning can improve non-linear models. Specifically, we replace the linear model components of \textsc{LegendreTron} and multinomial logistic regression with a VGG-5 architecture; which we refer to as \textsc{VGGTron} and VGG respectively in Table \ref{table:vgg_acc}. Our results show that learning proper losses can improve the performance of non-linear models. A more comprehensive survey of how loss learning can improve model performance for classification tasks is an avenue for future work, due to the variety of architectures and datasets.

\begin{table}
  \caption{Test classification accuracies of VGGTron and VGG for the \emph{MNIST}, \emph{Kuzushiji-MNIST} and \emph{Fashion-MNIST} datasets. See text for details.}
  \label{table:vgg_acc}
  \centering
  \begin{tabular}{llll}
    \toprule
         & MNIST & FMNIST & KMNIST\\
    \midrule
    \rowcolor[gray]{0.7}
    \textsc{VGGTron} & 99.59\%  & 92.88\% & 98.26\%\\
    VGG & 99.40\% & 92.80\%  & 98.12\%\\
    \bottomrule
  \end{tabular}
\end{table}

\paragraph{Other Multiclass Problems and Label Noise} We also compared \textsc{LegendreTron} against multinomial logistic regression on 15 datasets that are publicly available from the LIBSVM library \citep{chang2011libsvm}, the UCI machine learning repository \citep{asuncion2007uci, dua2017uci}, and the Statlog project \citep{king1995statlog}. We note that we did not compare our proposed method with other multiclass classification methods such as kernel methods explored in \citet{zien2007multiclass} and \citet{li2018multi}, as these methods are centred on the task of classification, whereas our focus is on jointly learning multiclass probabilities and proper canonical losses through the canonical link function. To assess the robustness against label noise, we also compare the classification accuracy of \textsc{LegendreTron} and multinomial logistic regression where labels in the training set are corrupted with probability $\eta$. That is, for any true label $y_n$, we instead train our models on the potentially corrupted label given by 
\begin{align*}
    \Tilde{y}_n = 
    \begin{cases}
        y_n &\text{with probability $1-\eta$},\\
        c &\text{with probability $\eta$ where $c\in \mathcal{Y} \setminus \{y_n\}$}
    \end{cases}.
\end{align*}
We applied \emph{symmetric} label noise in our experiments which is the case where the probability of $\Tilde{y}_n = c$ for each $c\in \mathcal{Y} \setminus \{y_n\}$ is $\frac{\eta}{C-1}$. We run both \textsc{LegendreTron} and multinomial logistic regression for each dataset 20 times, where each run randomly splits the dataset into 80\% training and 20\% testing sets. Our results in Table \ref{table:libsvm_results} show that \textsc{LegendreTron} outperforms multinomial logistic regression under a $t$-test at 99\% significance for most datasets and label noise settings. The performance of \textsc{LegendreTron} is on par with multinomial logistic regression on the \emph{svmguide2}, \emph{wine} and \emph{iris} datasets. Multinomial logistic regression only statistically outperforms \textsc{LegendreTron} on the \emph{dna} dataset. \textsc{LegendreTron} consistently outperforms multinomial logistic regression especially strongly on problems where the number of classes is greater than $10$. The better performance of \textsc{LegendreTron} can partially be attributed to greater model capacity afforded by $v^{-1}$ which allows logit estimates to adapt to problems with more classes by adding more nonlinearities. We note that the Lipschitz and strongly monotone properties of $\nabla g_1, \nabla g_2,\dots ,\nabla g_B$ are dependent only on inputs which remain uncorrupted so probability estimates would respect the true rankings of classes by design. We conjecture that these properties allow for more adaptive shrinking or expanding of logit variations depending on the level of label noise present; offering a form of tolerance to label noise. 

\begin{figure*}
  \centering
  \includegraphics{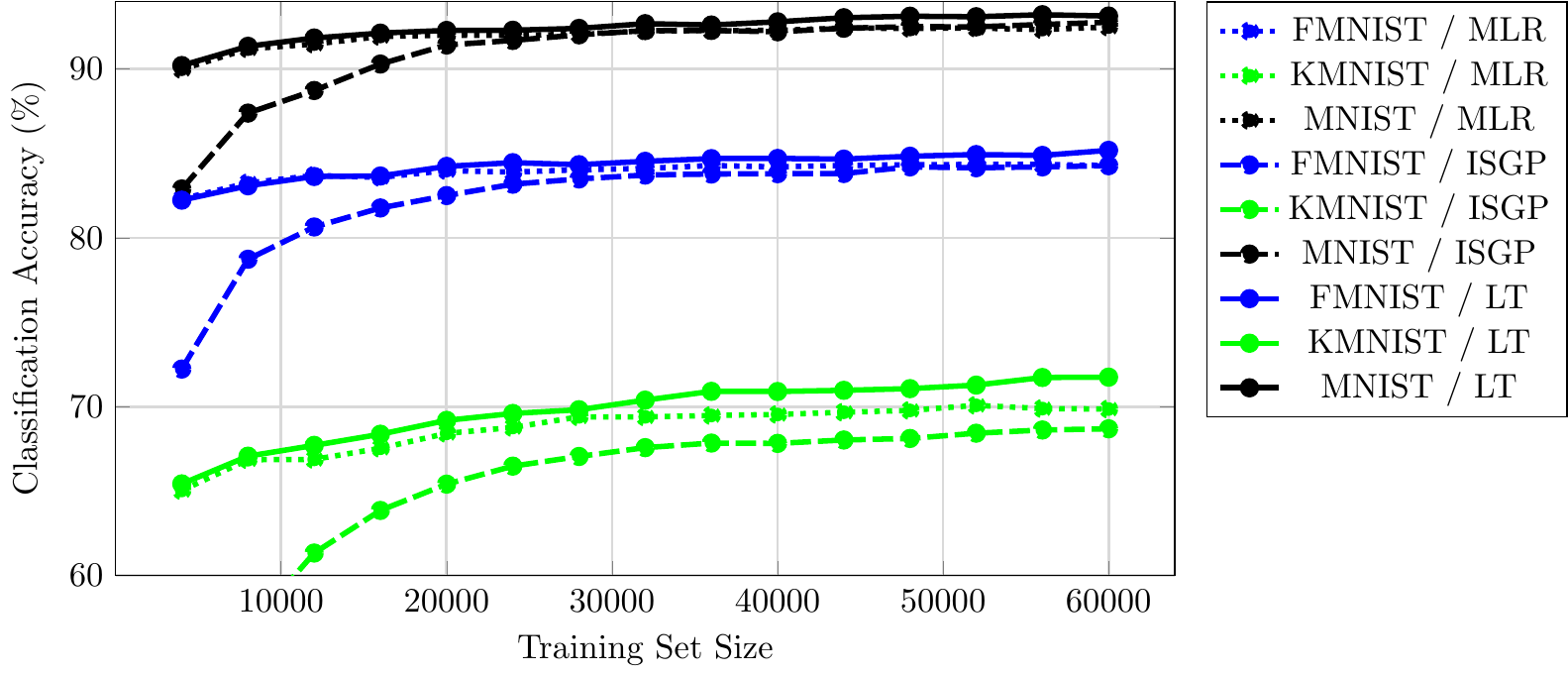}
  \caption{Test performance v.s. training set size for the \emph{MNIST}, \emph{Kuzushiji-MNIST} and \emph{Fashion-MNIST} datasets. We compare the ten-class classification accuracy of \textsc{LegendreTron} (LT), multinomial logistic regression (MLR) and ISGP-L\textsc{inkgistic} (ISGP) where the ISGP combines 10 one-vs-rest binary models while the former two algorithms model the probabilities of all 10 classes jointly.}
  \label{figure:classification_accuracy}
\end{figure*}

\begin{table*}
\caption{Average test classification accuracies (\%) for \textsc{LegendreTron} (LT) and multinomial logistic regression (MLR) on LIBSVM, UCI and Statlog datasets; at varying levels of label noise ($\eta$). Numbers of the method are bolded when it performs statistically better at a significance level of 99\% under a $t$-test. Absence of bolding indicates both methods have statistically similar performance.}
\label{table:libsvm_results}
\vskip 0.15in
\begin{center}
\begin{small}
\begin{tabular}{lcccccccc}
\toprule
    \multirow{2}{*}{Dataset} &
    \multirow{2}{*}{\# Features} &
    \multirow{2}{*}{\# Classes} &
    \multicolumn{2}{c}{$\eta=0\%$} &
    \multicolumn{2}{c}{$\eta=20\%$} &
    \multicolumn{2}{c}{$\eta=50\%$} \\
    & & & LT & MLR & LT & MLR & LT & MLR \\
\midrule
aloi & 128 & 1,000 & \textbf{88.11$\pm$0.03} & 10.34$\pm$0.42 & \textbf{83.03$\pm$0.06} & 7.07$\pm$0.45 & \textbf{75.23$\pm$0.07} & 3.53$\pm$0.29\\
sector & 55,197 & 105 & \textbf{89.71$\pm$0.18} & 8.77$\pm$0.73 & \textbf{81.00$\pm$0.28} & 4.12$\pm$0.44 & \textbf{57.38$\pm$0.31} & 3.17$\pm$0.47\\
letter & 16 & 26 & \textbf{79.82$\pm$0.30} & 53.37$\pm$0.25 & \textbf{74.17$\pm$0.21} & 51.24$\pm$0.28 & \textbf{64.28$\pm$0.26} & 46.78$\pm$0.41\\
news20 & 62,061 & 20 & \textbf{75.65$\pm$0.72} & 63.09$\pm$0.58 & \textbf{73.48$\pm$0.20} & 50.49$\pm$1.16 & \textbf{51.72$\pm$0.16} & 31.54$\pm$1.83\\
Sensorless & 48 & 11 & \textbf{88.31$\pm$0.19} & 34.42$\pm$0.46 & \textbf{82.63$\pm$0.99} & 32.70$\pm$0.50 & \textbf{52.02$\pm$0.78} & 29.35$\pm$0.84\\
vowel & 10 & 11 & \textbf{79.72$\pm$1.03} & 44.58$\pm$1.08 & \textbf{63.77$\pm$1.36} & 43.44$\pm$1.17 & \textbf{40.94$\pm$1.61} & 35.42$\pm$1.45\\
\midrule
usps & 256 & 10 & \textbf{95.23$\pm$0.16} & 93.79$\pm$0.17 & 92.88$\pm$0.15 & 92.95$\pm$0.19 & 90.23$\pm$0.26 & 90.48$\pm$0.27\\
segment & 19 & 7 & \textbf{95.95$\pm$0.24} & 87.86$\pm$0.40 & \textbf{92.21$\pm$0.40} & 87.28$\pm$0.40 & \textbf{86.56$\pm$0.47} & 82.75$\pm$0.46\\
satimage & 36 & 6 & \textbf{86.97$\pm$0.19} & 83.93$\pm$0.28 & \textbf{84.93$\pm$0.25} & 81.16$\pm$0.28 & 77.44$\pm$0.29 & 77.39$\pm$0.29\\
glass & 36 & 6 & 58.72$\pm$1.94 & 52.09$\pm$1.88 & 53.72$\pm$1.98 & 50.47$\pm$2.11 & 42.56$\pm$1.92 & 45.47$\pm$1.67\\
vehicle & 18 & 4 & \textbf{76.91$\pm$0.65} & 64.94$\pm$0.43 & \textbf{73.59$\pm$0.79} & 63.06$\pm$0.53 & \textbf{60.94$\pm$1.25} & 55.18$\pm$1.20\\
dna & 180 & 3 & 92.79$\pm$0.30 & \textbf{94.43$\pm$0.19} & 82.61$\pm$0.51 & \textbf{89.55$\pm$0.31} & 58.23$\pm$1.05 & \textbf{64.18$\pm$0.81}\\
svmguide2 & 20 & 3 & 56.01$\pm$1.40 & 56.01$\pm$1.40 & 56.01$\pm$1.40 & 56.01$\pm$1.40 & 51.65$\pm$2.81 & 52.41$\pm$3.04\\
wine & 13 & 3 & 96.94$\pm$1.14& 97.78$\pm$0.59 & 90.97$\pm$1.92 & 96.25$\pm$0.99 & 69.44$\pm$2.89 & 77.36$\pm$2.46\\
iris & 4 & 3 & 86.67$\pm$3.89 & 83.00$\pm$2.08 & 80.00$\pm$3.71 & 81.50$\pm$2.27 & 63.50$\pm$5.13 & 70.67$\pm$3.83\\
\bottomrule
\end{tabular}
\end{small}
\end{center}
\vskip -0.1in
\end{table*}

\smallskip
\section{Conclusion and Broader Impact}
\label{section:conclusion}
In this work, we proposed a general approach which jointly learns proper canonical losses and multiclass probabilities. Our contributions advance the recent work on learning losses with probabilities based on the seminal work within \citet{kakade2011efficient, nock2020supervised, walder2020all} by providing a natural extension to the multiclass setting. The practical nature and generality of our model is owed to the general parameterisation of Fully Input Convex Neural Networks, with theoretical support from Legendre functions, structures from information geometry and hallmark results from convex analysis. 

By grounding losses in properness for the multiclass setting, we have demonstrated that our model improves upon existing methods that aim to solve multiclass problems through binary reductions, and also outperforms the natural baseline of multinomial logistic regression. Separately, we have also provided conditions under which a composition of gradients of differentiable convex functions is the gradient of another differentiable convex function. 

While it is possible for advances in machine learning to bring positive and negative societal impacts, the present work remains general and not specific to any application so it is unlikely to bring about any immediate negative societal impact. We anticipate that our results will find applications in multiclass classification and probability estimation, as well as variational inference.

\bibliographystyle{plainnat}
\bibliography{main.bib}

\newpage
\appendix
\onecolumn
\section{Convex Analysis: Relevant Background and List of Theorems}
\label{appendix:convex_analysis}
\subsection{Background}
To motivate the results studied in this section, we first note that in general, the composition of two monotone functions in $\R^{C-1}$ is not necessarily another monotone function in $\R^{C-1}$. This means that methods to design monotonic functions in $\R$ cannot be applied to functions defined on $\R^{C-1}$, leaving the methods discussed in Section \ref{section:related_work} unsuitable for the general multiclass setting. Separately, we note that the composition of two gradients of differentiable convex functions is \emph{not} necessarily the gradient of another convex function. In general, to claim that a function $f:\R^{C-1} \to \R^{C-1}$ is the gradient of a convex function $g:\R^{C-1} \to \R$, requires $f$ to satisfy a notion of monotonicity generalised to higher dimensions. The connection between convex functions and their gradients is well known in convex analysis via the notion of \emph{maximal cyclically monotone} functions. This is a combination of two notions of monotonicity: \emph{maximal monotonicity} and \emph{cyclical monotonicity}. These are defined within the following list of definitions and theorems.
\subsection{List of Theorems}
\begin{lemma}
\label{lemma:definiteness_eigenvalues}
Let $A \in \R^{(C-1) \times (C-1)}$ be a square symmetric matrix. Then $A$ has $C-1$ real eigenvalues $\lambda_i$ for $i=1,\dots, C-1$. Moreover, $A$ is
\begin{itemize}
    \item positive definite if and only if $\lambda_i>0$ for $i=1,\dots, C-1$.
    \item positive semi-definite if and only if $\lambda_i\geq0$ for $i=1,\dots, C-1$.
    \item negative definite if and only if $\lambda_i<0$ for $i=1,\dots, C-1$.
    \item negative semi-definite if and only if $\lambda_i\leq0$ for $i=1,\dots, C-1$.
    \item indefinite if there exist $i,j \in \{1,\dots,C-1\}$ such that $\lambda_i >0$ and $\lambda_j<0$.
\end{itemize}
\end{lemma}
\begin{definition}[{\citep[Definition 12.1]{rockafellar2009variational}}]
\label{definition:monotone}
A function $f:\R^{C-1}\to \R^{C-1}$ is monotone if $\langle f(\mathbf{x}) - f(\mathbf{z}), \mathbf{x} - \mathbf{z}\rangle \geq 0$ for all $\mathbf{x},\mathbf{z} \in \R^{C-1}$. Moreover, it is strictly monotone when the inequality is strict whenever $\mathbf{x} \neq \mathbf{z}$.
\end{definition}
\begin{corollary}
\label{corollary:strictly_monotone_invertible}
Let $f:\R^{C-1}\to \R^{C-1}$ be a strictly monotone function. Then $f$ is invertible.
\end{corollary}
\begin{proof}
Suppose $f$ is strictly monotone and assume for a proof by contradiction that $f$ is not invertible. Then there exists $\mathbf{x},\mathbf{z} \in \R^{C-1}$ such that $f(\mathbf{x})=f(\mathbf{z})$. That is, we have $\mathbf{x}-\mathbf{z} \neq \mathbf{0}$ and $f(\mathbf{x})-f(\mathbf{z}) =\mathbf{0}$. This implies that 
\begin{align*}
    \langle f(\mathbf{x}) - f(\mathbf{z}), \mathbf{x} - \mathbf{z}\rangle = 0.
\end{align*}
This is a contradiction since $f$ is strictly monotone. Thus, $f$ must be invertible.
\end{proof}

The following two definitions require the notion of the graph of a function $f:\R^{C-1}\to \R^{C-1}$ which is defined as $\text{gph}(f) = \{(\mathbf{x},\mathbf{y}): \mathbf{x} \in \R^{C-1}, \mathbf{y} \in f(\mathbf{x})\}$.
\begin{definition}[{\citep[Definition 20.20]{bauschke2011convex}}]
\label{definition:maximal_monotone}
Let $f:\R^{C-1}\to \R^{C-1}$ be a monotone function. Then $f$ is maximally monotone if there exists no monotone function $g:\R^{C-1}\to \R^{C-1}$ such that $\text{gph}(f) \subsetneq \text{gph}(g)$.
\end{definition}

\begin{definition}[{\citep[Definition 22.10]{bauschke2011convex}}]
\label{definition:maximal_cyclical_monotone}
Let $f:\R^{C-1} \to \R^{C-1}$. For an arbitrary integer $n\geq 2$, $f$ is $n$-cyclically monotone if for any $\{(\mathbf{x}_i,\mathbf{y}_i)\}_{i=1,\dots,n} \subset \text{gph}(f)$ it follows that
\begin{align*}
    \sum_{i=1}^n\langle \mathbf{y}_i, \mathbf{x}_{i+1} - \mathbf{x}_i \rangle \leq 0 \text{ where $\mathbf{x}_{n+1}= \mathbf{x}_1$}.
\end{align*}
$f$ is cyclically monotone if it is $n$-cyclically monotone for any integer $n\geq 2$. In addition, if $\text{gph}(f) \not\subset \text{gph}(g)$ for any cyclically monotone function $g\neq f$ then $f$ is maximal cyclically monotone.
\end{definition}

\begin{theorem}[{\citep[Theorems 12.17 \& 12.25]{rockafellar2009variational}}]
\label{lemma:maximal_cyclic_monotone_iff_gradient_convex}
Let $f:\R^{C-1} \to \R^{C-1}$. Then $f = \nabla h$ for a differentiable convex function $h:\R^{C-1} \to \R$ if and only if $f$ is maximal cyclically monotone. That is, $f$ is maximally monotone and cyclically monotone.
\end{theorem}

\begin{theorem}[{\citep[Proposition 12.3]{rockafellar2009variational}}]
\label{theorem:monotone_iff_jacobian_psd}
Let $f:\R^{C-1} \to \R^{C-1}$ be a differentiable function. Then $f$ is monotone if and only if $\nabla f(\mathbf{x})$ is positive semi-definite for all $\mathbf{x} \in \R^{C-1}$. Moreover, $f$ is strictly monotone if $\nabla f(\mathbf{x})$ is positive definite for all $\mathbf{x} \in \R^{C-1}$.
\end{theorem}
\begin{theorem}[{\citep[Corollary 20.25]{bauschke2011convex}}]
\label{theorem:continuous_monotone_implies_maximally_monotone}
Let $f:\R^{C-1} \to \R^{C-1}$ be a monotone and continuous function. Then $f$ is maximally monotone.
\end{theorem}
\begin{theorem}[{\citep[Theorem 3]{borwein2007decomposition}}]
\label{theorem:asplund_composition_maximal_monotone}
Let $f:\R^{C-1} \to \R^{C-1}$ be maximally monotone and continuously differentiable. Then $f(\mathbf{x}) = \nabla F(\mathbf{x}) + L\mathbf{x}$ where $F$ is a differentiable convex function, and $L$ is a skew symmetric matrix. 
\end{theorem}
\begin{theorem}[{\citep[Theorem 3]{meenakshi1999product}}]
\label{theorem:matrix_products_psd_iff_symmetric}
Let $A,B \in \R^{(C-1) \times (C-1)}$ be symmetric and positive semi-definite matrices. Then $AB$ is positive semi-definite if and only if it is symmetric.
\end{theorem}

\begin{theorem}[{\citep[Theorem VIII.4.6]{bhatia2013matrix}}]
\label{theorem:eigenvalues_matrix_sum}
Let $V \subset \R^{(C-1)\times (C-1)}$ be a real vector space whose elements are matrices with real eigenvalues. Denote $\lambda_i(M)$ as the $i$-th smallest eigenvalue for any matrix $M\in V$. Let $A,B \in V$ then
\begin{align*}
    \lambda_i(A) + \lambda_1(B) \leq \lambda_i(A+B) \leq \lambda_i(A) + \lambda_{C-1}(B).
\end{align*}
\end{theorem}

\subsection{Remarks}
Theorem \ref{lemma:maximal_cyclic_monotone_iff_gradient_convex} serves as a criterion and characterisation of differentiable convex functions through their gradients. Theorems \ref{theorem:continuous_monotone_implies_maximally_monotone} to \ref{theorem:matrix_products_psd_iff_symmetric} are hallmark results from the rich literature of convex analysis and monotone operators that tie together conditions under which a differentiable composite function is the gradient of a convex function. Notably, Theorem \ref{theorem:asplund_composition_maximal_monotone} is a rewritten version of the Asplund decomposition of maximal monotone operators \citep{asplund1970decomposition} which tells us it suffices to focus on maximal monotonicity. We refer the reader to Appendix \ref{appendix:composition_iff_psd} for the usage of Theorems \ref{theorem:monotone_iff_jacobian_psd} to \ref{theorem:matrix_products_psd_iff_symmetric} in the proof of Theorem \ref{theorem:composition_gradient_iff_psd}.

Theorem \ref{theorem:eigenvalues_matrix_sum} allows us to obtain a lower bound on the smallest eigenvalue of the sum of two real-valued matrices with real eigenvalues. This is particularly useful to prove positive definiteness in Theorem \ref{theorem:maps_with_pd_jacobians_closed_under_composition}. We refer the reader to Appendix \ref{appendix:maps_with_pd_jacobians_closed_under_composition} for its usage in the proof of Theorem \ref{theorem:maps_with_pd_jacobians_closed_under_composition}.

\section{Examples of Binary Proper Losses}
Denote $y\in \{-1,1\}$ a label and $p=\Pr(Y=1|\mathbf{x})$ be the true probability that $Y=1|\mathbf{x}$. Let $\hat{y}$ and $\hat{p}$ be the predicted class and probability estimate of $Y=1$ given input $\mathbf{x}$. Table \ref{table:proper_loss_examples} and Figure \ref{figure:proper_loss_examples} show the formulas and plots of the partial losses that correspond to various examples of proper losses.

\begin{table*}
\caption{Analytical formulas of partial losses for various examples of binary proper losses where it is assumed that $p=1$.}
\label{table:proper_loss_examples}
\vskip 0.15in
\begin{center}
\begin{small}
\begin{tabular}{lc}
\toprule
 & $\ell_1(q)$\\
\midrule
$0-1$ & $[\![\hat{y}(q)=1]\!]$\\
square & $(1-q)^2$\\
log & $-\log(q)$ \\
Matsushita & $\frac{1}{2}\sqrt{\frac{1-q}{q}}$ \\
\bottomrule
\end{tabular}
\end{small}
\end{center}
\vskip -0.1in
\end{table*}

\begin{figure}
  \centering
  \includegraphics{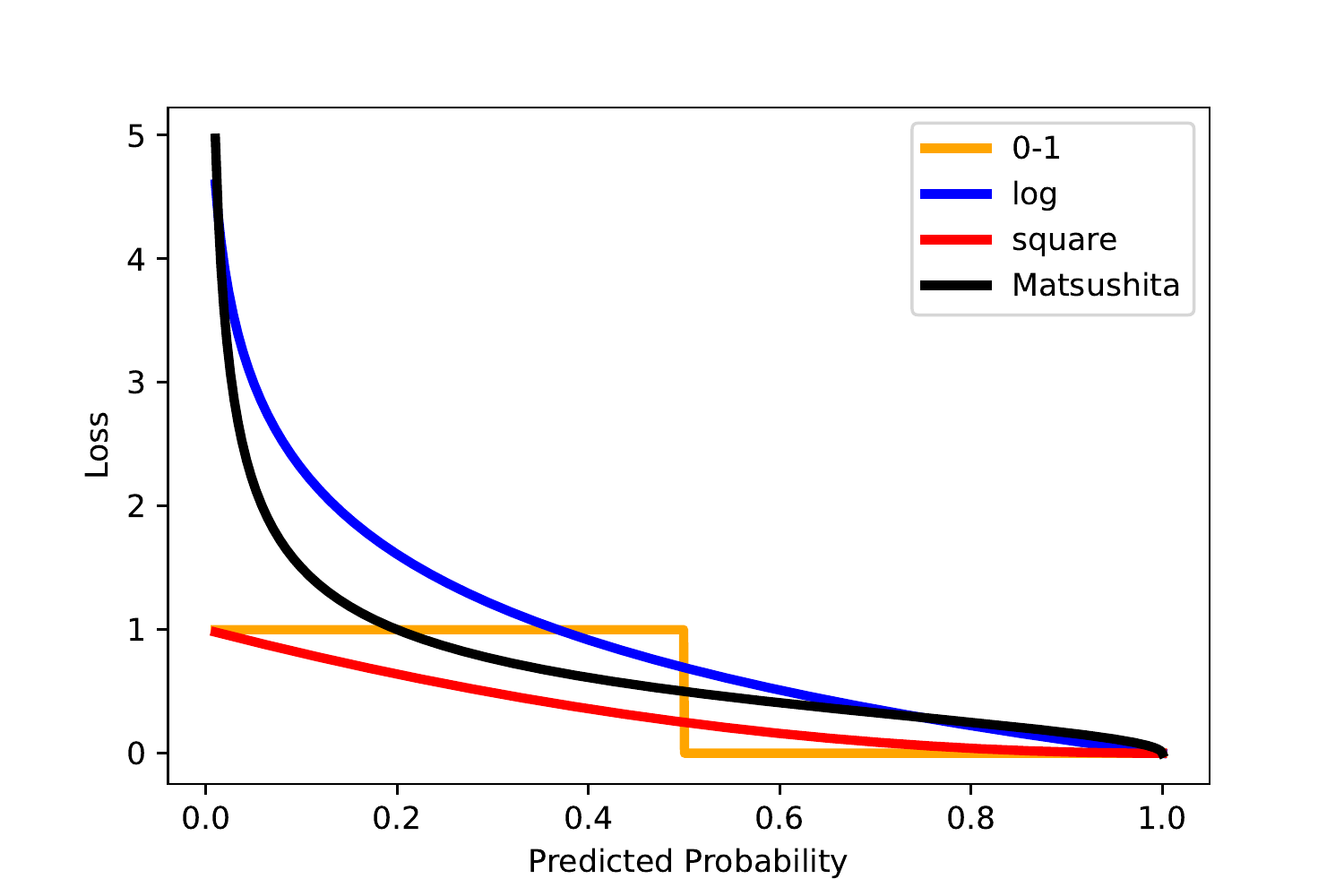}
  \caption{Plots of partial losses corresponding various proper losses when true probability is 1.}
  \label{figure:proper_loss_examples}
\end{figure}

\section{Proof of equivalent conditions on subdifferentials for strictly convex functions}
($\Rightarrow$) Suppose $f$ is strictly convex and assume for a proof by contradiction that there exists some $\mathbf{x}, \mathbf{y} \in \text{dom} f$ such that $\mathbf{x} \neq \mathbf{y}$ with $f(\mathbf{x})+\langle\phi, \mathbf{y}-\mathbf{x}\rangle \geq f(\mathbf{y})$ for some $\phi \in \partial f(\mathbf{x})$. 

Fix $\lambda \in (0,1)$. Then we have
\begin{align*}
    f(\mathbf{x}) + \langle\phi, (\lambda \mathbf{x} +(1-\lambda)\mathbf{y})-\mathbf{x}\rangle &=f(\mathbf{x}) + (1-\lambda)\langle\phi, \mathbf{y}-\mathbf{x}\rangle\\
    &\leq f(\lambda \mathbf{x} +(1-\lambda)\mathbf{y}) \text{ by definition of a subgradient}\\
    &< \lambda f(\mathbf{x}) + (1-\lambda)f(\mathbf{y}) \text{ by strict convexity of $f$}\\
    &\leq f(\mathbf{x}) + (1-\lambda)\langle\phi, \mathbf{y}-\mathbf{x}\rangle \text{ by the above assumption}.
\end{align*}
Thus, we have a contradiction so we must have the subdifferential of $f$ for all $\mathbf{x} \in \text{dom} f$ is given by
\begin{align*}
    \partial f(\mathbf{x})=\left\{\phi \in \mathbb{R}^{n}:\langle\phi, \mathbf{y}-\mathbf{x}\rangle < f(\mathbf{y})-f(\mathbf{x}),\forall \mathbf{y} \in \mathbb{R}^{n}\right\}.
\end{align*}

($\Leftarrow$) Suppose the subdifferential of $f$ for any $\mathbf{x} \in \text{dom} f$ is given by
\begin{align*}
    \partial f(\mathbf{x})=\left\{\phi \in \mathbb{R}^{n}:\langle\phi, \mathbf{y}-\mathbf{x}\rangle < f(\mathbf{y})-f(\mathbf{x}), \forall \mathbf{y} \in \mathbb{R}^{n}\right\}.
\end{align*}

Fix $\mathbf{x},\mathbf{y}\in \text{dom} f$ and $\lambda \in (0,1)$. Consider $\phi \in \partial f(\lambda \mathbf{x} + (1-\lambda)\mathbf{y})$. Then we have 
\begin{align*}
    f(\lambda \mathbf{x} + (1-\lambda)\mathbf{y})+(1-\lambda)\langle\phi, \mathbf{x}-\mathbf{y}\rangle < f(\mathbf{x}),\\
    f(\lambda \mathbf{x} + (1-\lambda)\mathbf{y})+\lambda\langle\phi, \mathbf{y}-\mathbf{x}\rangle < f(\mathbf{y}).
\end{align*}
Multiplying the first inequality by $\lambda$ and the second by $(1-\lambda)$, summing them gives us $f(\lambda \mathbf{x} + (1-\lambda)\mathbf{y}) < \lambda f(\mathbf{x}) + (1-\lambda)f(\mathbf{y})$. This holds for arbitrary $\mathbf{x},\mathbf{y}\in \text{dom} f$ and $\lambda \in (0,1)$ so it follows that $f$ is strictly convex.

\section{Proof of Proposition \ref{proposition:proper_loss_bregman}}
\label{appendix:proper_loss_bregman}
($\Rightarrow$) Fix $q\in \Delta^{C-1}$. Suppose $\ell$ is proper. Then we have
\begin{align*}
    L(p,q)=p^{\top}\ell(q) =  q^{\top}\ell(q) + (p-q)^{\top}\ell(q)= \underline{L}(q) + (p-q)^{\top}\ell(q)
\end{align*}
and also,
\begin{align*}
    0\leq L(p,q) - L(p,p) &= \underline{L}(q) + (p-q)^{\top}\ell(q) - \underline{L}(p)\\
    \implies -(p-q)^{\top}\ell(q) &\leq - \underline{L}(p) - (-\underline{L}(q)).
\end{align*}
Recall that $\underline{L}$ is concave so it follows that $-\underline{L}$ is convex. Hence, $-\ell(q) \in \partial (-\underline{L})(q)$ which means $-\ell(q)$ is a subgradient of $-\underline{L}$ at $q$ and $L(p,q)= -(-\underline{L}(q)) - (p-q)^{\top}(-\ell(q))$.

($\Leftarrow$) Suppose there exists a convex function $f:\Delta^{C-1} \to \R$ such that for all $q\in \Delta^{C-1}$, there exists a subgradient $\phi \in \partial f(q)$ and $L(p,q)= -f(q) - (p-q)^{\top}\phi$.

For all $p\in \Delta^{C-1}$, we have
\begin{align*}
    L(p,q)-L(p,p) &=  f(p)-f(q) - (p-q)^{\top}\phi\\
    &\geq 0 \text{ since $\phi$ is a subgradient of $f$ at $q$}\\
    \implies L(p,p) &\leq L(p,q).
\end{align*}
Hence, $\ell$ is a proper loss.

To prove that $\ell$ is strictly proper if and only if there exists a strictly convex function $f:\Delta^{C-1} \to \R$ such that for all $q\in \Delta^{C-1}$, there exists a subgradient $\phi \in \partial f(q)$ such that $L(p,q) = -(p-q)^{\top} \phi - f(q) \text{ for all $p \in \Delta^{C-1}$}$. This follows immediately by definitions of strictly proper losses and subdifferentials from which the above inequalities become strict.

We are left to prove that $L(p,q) = (p-q)^{\top} \ell(q) + \underline{L}(q)$ when $\underline{L}$ is differentiable. We first note that $\underline{L}$ is concave so $-\underline{L}$ is convex. Recall that $L(p,q)=p^{\top}\ell(q) = \underline{L}(q) + (p-q)^{\top}\ell(q)$ from the workings within Appendix \ref{appendix:proper_loss_bregman}. Setting $f=-\underline{L}$ for Proposition \ref{proposition:proper_loss_bregman}, we can deduce $-\ell(q) = -\nabla \underline{L}(q) \forall q \in \text{ri}(\Delta^{C-1})$ for a proper loss $\ell$ which follows by the uniqueness of subgradients for differentiable functions. That is, $\nabla \underline{L}(q) = \ell(q), \forall q \in \text{ri}(\Delta^{C-1})$.

\section{Convexity and Expression of Proper Canonical Losses}
\label{appendix:convexity_proper_canonical_losses}
In this section, we first show that any differentiable proper loss can be transformed such that its partial losses are convex functions. This is done by equipping a proper loss $\ell$ with its corresponding canonical link $\Tilde{\psi}$ to form $\Tilde{\ell}(\mathbf{x})= \bigl(\ell \circ \Pi^{-1} \circ \Tilde{\psi}^{-1}\bigr)(\mathbf{x})$. The proof of this result has been collated from \citet{erven2012mix,reid2012design,williamson2016multiclass} and is included for completeness. We then conclude this section by extending this result by providing an analytical expression for a proper canonical loss.

We write $\ell(p) = (\ell_{-C}(p), \ell_C(p))$ where $\ell_{-C}(p) \in \R^{C-1}$ and $\ell_{C}(p) \in \R$ denote the first $C-1$ components of $\ell(p)$ and the last component of $\ell(p)$ respectively.
\begin{proposition}
\label{proposition:proper_loss_stationarity}
Let $\ell:\Delta^{C-1} \to \R^C$ be a differentiable proper loss. Then 
\begin{align*}
    J_{\ell_{C} \circ \Pi^{-1}}(\Tilde{p}) = -\frac{\Tilde{p}^{\top}}{p_C}J_{\ell_{-C}\circ \Pi^{-1}}(\Tilde{p})
\end{align*}
where $J_{\ell_{C} \circ \Pi^{-1}}$ and $J_{\ell_{-C} \circ \Pi^{-1}}$ are the Jacobians of $\ell_{C} \circ \Pi^{-1}$ and $\ell_{-C} \circ \Pi^{-1}$ respectively.
\end{proposition}
\begin{proof}
Fix $p \in \Delta^{C-1}$ and consider $q \in \Delta^{C-1}$. Let $\Tilde{p} = \Pi(p)$ and $\Tilde{q} = \Pi(q)$. We have
\begin{align*}
    L(p,q) &= p^{\top}\ell(q) \\
    &= \Tilde{p}^{\top}\ell_{-C}(q) + p_C \ell_{C}(q)\\
    &= \Tilde{p}^{\top}(\ell_{-C} \circ \Pi^{-1})(\Tilde{q}) + p_C (\ell_{C} \circ \Pi^{-1})(\Tilde{q}).
\end{align*}
We can define the above as a function of $\Tilde{q}$. That is, 
\begin{align*}
    L_p(\Tilde{q}) = \Tilde{p}^{\top}(\ell_{-C} \circ \Pi^{-1})(\Tilde{q}) + p_C (\ell_{C} \circ \Pi^{-1})(\Tilde{q}).
\end{align*}
$L_p(\Tilde{q})$ is differentiable with its Jacobian given by
\begin{align*}
    J_{L_p}(\Tilde{q}) = \Tilde{p}^{\top}J_{\ell_{-C} \circ \Pi^{-1}}(\Tilde{q}) + p_C J_{\ell_{C} \circ \Pi^{-1}}(\Tilde{q}).
\end{align*}
Since $\ell$ is proper, $L_p(\Tilde{q})$ has an minimum at $\Tilde{q} = \Tilde{p}$ with its Jacobian satisfying the stationarity condition 
\begin{align*}
    \Tilde{p}^{\top}J_{\ell_{-C} \circ \Pi^{-1}}(\Tilde{p}) + p_C J_{\ell_{C} \circ \Pi^{-1}}(\Tilde{p}) = \mathbf{0}_{C-1}^{\top}.
\end{align*}
Rearranging the above equality completes the proof.
\end{proof}

Proposition \ref{proposition:proper_loss_stationarity} allows to express the Jacobian of the $C$-th partial loss in terms of the Jacobian of the first $C-1$ partial losses. This result allows us to express the Jacobian and Hessian of the projected conditional Bayes risk in the following proposition.

\begin{proposition}
\label{proposition:bayes_risk_jacobian_hessian}
Let $\ell:\Delta^{C-1} \to \R^C$ be a differentiable proper loss and $\Tilde{\underline{L}}: \Tilde{\Delta}^{C-1} \to \R$ be its associated projected conditional Bayes risk. Then the Jacobian and Hessian of $\Tilde{\underline{L}}$ are given by
\begin{align*}
    J_{\Tilde{\underline{L}}} = \bigl((\ell_{-C} \circ \Pi^{-1})(\Tilde{p}) \bigr)^{\top} - \bigl((\ell_{C} \circ \Pi^{-1})(\Tilde{p})\bigr)\mathbbm{1}_{C-1}^{\top}
\end{align*}
and 
\begin{align*}
    H_{\Tilde{\underline{L}}} = \left(I_{C-1} + \mathbbm{1}_{C-1}\frac{\Tilde{p}^{\top}}{p_C}\right)J_{\ell_{-C} \circ \Pi^{-1}}(\Tilde{p}).
\end{align*}
\end{proposition}
\begin{proof}
Consider $p \in \Delta^{C-1}$ and let $\Tilde{p} = \Pi(p)$. Following the workings of Proposition \ref{proposition:proper_loss_stationarity}, we have
\begin{align*}
    \Tilde{\underline{L}}(\Tilde{p}) &= L(p,p) \\
    &= \Tilde{p}^{\top}(\ell_{-C} \circ \Pi^{-1})(\Tilde{p}) + p_C(\Tilde{p}) (\ell_{C} \circ \Pi^{-1})(\Tilde{p}).
\end{align*}
Using the product rule and noting $J_{p_C}(\Tilde{p}) = -\mathbbm{1}_{C-1}^{\top}$, the Jacobian of $\Tilde{\underline{L}}$ is given by
\begin{align*}
    J_{\Tilde{\underline{L}}}(\Tilde{p}) &= \Tilde{p}^{\top} J_{\ell_{-C} \circ \Pi^{-1}}(\Tilde{p}) + \bigl((\ell_{-C} \circ \Pi^{-1})(\Tilde{p}) \bigr)^{\top} \\
    & \quad + (\ell_{C} \circ \Pi^{-1})(\Tilde{p})J_{p_C}(\Tilde{p}) + p_C(\Tilde{p}) J_{\ell_{C} \circ \Pi^{-1}}(\Tilde{p})\\
    &= \Tilde{p}^{\top} J_{\ell_{-C} \circ \Pi^{-1}}(\Tilde{p}) + \bigl((\ell_{-C} \circ \Pi^{-1})(\Tilde{p}) \bigr)^{\top} \\
    & \quad - \bigl((\ell_{C} \circ \Pi^{-1})(\Tilde{p})\bigr)\mathbbm{1}_{C-1}^{\top} + p_C(\Tilde{p}) J_{\ell_{C} \circ \Pi^{-1}}(\Tilde{p})\\
    &=\bigl((\ell_{-C} \circ \Pi^{-1})(\Tilde{p}) \bigr)^{\top} - \bigl((\ell_{C} \circ \Pi^{-1})(\Tilde{p})\bigr)\mathbbm{1}_{C-1}^{\top} \text{ using Proposition \ref{proposition:proper_loss_stationarity}}.
\end{align*}
Differentiating $\bigl(J_{\Tilde{\underline{L}}}(\Tilde{p})\bigr)^{\top} = (\ell_{-C} \circ \Pi^{-1})(\Tilde{p}) - \bigl((\ell_{C} \circ \Pi^{-1})(\Tilde{p})\bigr)\mathbbm{1}_{C-1}$, the Hessian of $\Tilde{\underline{L}}$ is given by
\begin{align*}
    H_{\Tilde{\underline{L}}}(\Tilde{p}) &= J_{\ell_{-C} \circ \Pi^{-1}}(\Tilde{p}) - \mathbbm{1}_{C-1}J_{\ell_{C} \circ \Pi^{-1}}(\Tilde{p})\\
    &= J_{\ell_{-C} \circ \Pi^{-1}}(\Tilde{p}) + \mathbbm{1}_{C-1}\frac{\Tilde{p}^{\top}}{p_C}J_{\ell_{-C} \circ \Pi^{-1}}(\Tilde{p}) \text{ using Proposition \ref{proposition:proper_loss_stationarity}}\\
    &= \left(I_{C-1} + \mathbbm{1}_{C-1}\frac{\Tilde{p}^{\top}}{p_C}\right)J_{\ell_{-C} \circ \Pi^{-1}}(\Tilde{p}).
\end{align*}
\end{proof}

Proposition \ref{proposition:bayes_risk_jacobian_hessian} provides expressions for us to establish the connection between differentiable proper losses and their corresponding canonical links.
\begin{corollary}
\label{corollary:connection_links_losses}
Let $\ell:\Delta^{C-1} \to \R^C$ be a differentiable proper loss and $\Tilde{\underline{L}}: \Tilde{\Delta}^{C-1} \to \R$ be its associated projected conditional Bayes risk. Then we have 
\begin{align*}
    \Tilde{\psi}(\Tilde{p}) = \bigl((\ell_{C} \circ \Pi^{-1})(\Tilde{p})\bigr)\mathbbm{1}_{C-1} - (\ell_{-C} \circ \Pi^{-1})(\Tilde{p}).
\end{align*}
\end{corollary}
\begin{proof}
    From the definition of canonical links, we have
    \begin{align*}
        \Tilde{\psi}(\Tilde{p}) &= -\nabla\Tilde{\underline{L}}(\Tilde{p})\\
        &=-\bigl(J_{\Tilde{\underline{L}}}(\Tilde{p})\bigr)^{\top}\\
        &= \bigl((\ell_{C} \circ \Pi^{-1})(\Tilde{p})\bigr)\mathbbm{1}_{C-1} - (\ell_{-C} \circ \Pi^{-1})(\Tilde{p}).
    \end{align*}
\end{proof}

Proposition \ref{proposition:bayes_risk_jacobian_hessian} also allows us to formulate the Jacobian of $\ell \circ \Pi^{-1}$ in the following corollary.
\begin{corollary}
\label{corollary:complete_loss_jacobian}
Let $\ell:\Delta^{C-1} \to \R^C$ be a differentiable proper loss and $\Tilde{\underline{L}}: \Tilde{\Delta}^{C-1} \to \R$ be its associated projected conditional Bayes risk. Then we have 
\begin{align*}
    J_{\ell \circ \Pi^{-1}}(\Tilde{p}) = 
    \begin{bmatrix}
    I_{C-1} - \mathbbm{1}_{C-1}\Tilde{p}^{\top}\\
    - \frac{\Tilde{p}^{\top}}{p_C} (I_{C-1} - \mathbbm{1}_{C-1}\Tilde{p}^{\top})
    \end{bmatrix}
    H_{\Tilde{\underline{L}}}(\Tilde{p}).
\end{align*}
\end{corollary}
\begin{proof}
For any $\Tilde{p} \in \Tilde{\Delta}^{C-1}$ and $p_C = 1- \mathbbm{1}_{C-1}^{\top}\Tilde{p}$, we first note that 
\begin{align*}
    (I_{C-1} - \mathbbm{1}_{C-1}\Tilde{p}^{\top})\left(I_{C-1} + \mathbbm{1}_{C-1}\frac{\Tilde{p}^{\top}}{p_C}\right) = I_{C-1}.
\end{align*}
Using the above with Proposition \ref{proposition:bayes_risk_jacobian_hessian}, gives us
\begin{align*}
    J_{\ell_{-C} \circ \Pi^{-1}}(\Tilde{p}) &= (I_{C-1} - \mathbbm{1}_{C-1}\Tilde{p}^{\top}) H_{\Tilde{\underline{L}}}(\Tilde{p}).
\end{align*}
Proposition \ref{proposition:proper_loss_stationarity} then gives us
\begin{align*}
    J_{\ell_{C} \circ \Pi^{-1}}(\Tilde{p}) &= -\frac{\Tilde{p}^{\top}}{p_C}J_{\ell_{-C}\circ \Pi^{-1}}(\Tilde{p})\\
    &=-\frac{\Tilde{p}^{\top}}{p_C}(I_{C-1} - \mathbbm{1}_{C-1}\Tilde{p}^{\top}) H_{\Tilde{\underline{L}}}(\Tilde{p}).
\end{align*}
$J_{\ell_{-C} \circ \Pi^{-1}}(\Tilde{p})$ and $J_{\ell_{C} \circ \Pi^{-1}}(\Tilde{p})$ form the upper $(C-1) \times (C-1)$ matrix and the lower $1 \times (C-1)$ matrix of $J_{\ell \circ \Pi^{-1}}(\Tilde{p})$ respectively. This completes the proof.
\end{proof}

With the expression of the Jacobian $J_{\ell \circ \Pi^{-1}}(\Tilde{p})$ in hand, we can deduce that $\Tilde{\ell}_i(\mathbf{x})= \bigl(\ell_i \circ \Pi^{-1} \circ \Tilde{\psi}^{-1}\bigr)(\mathbf{x})$ is convex for each $i=1,\dots,C$. In other words, each component of $\Tilde{\ell}(\mathbf{x})$ is convex.

\begin{theorem}
\label{theorem:proper_canonical_losses_convex}
Let $\Tilde{\underline{L}}: \Tilde{\Delta}^{C-1} \to \R$ be a twice-differentiable projected conditional Bayes risk, $\ell:\Delta^{C-1} \to \R^C$ be its associated proper loss and $\Tilde{\psi}: \Tilde{\Delta}^{C-1} \to \R^{C-1}$ be its associated canonical link. Suppose $H_{-\Tilde{\underline{L}}}(\Tilde{p})$ is positive definite for all $\Tilde{p} \in \Tilde{\Delta}^{C-1}$. Then 
\begin{align*}
    \Tilde{\ell}_i(\mathbf{x})= \bigl(\ell_i \circ \Pi^{-1} \circ \Tilde{\psi}^{-1}\bigr)(\mathbf{x})
\end{align*}
is convex for each $i=1,\dots,C$.
\end{theorem}
\begin{proof}
Recall that $\Tilde{\underline{L}}$ is concave. It follows that $-\Tilde{\underline{L}}$ is a twice-differentiable convex function.

Consider $\mathbf{x} \in \R^{C-1}$ and denote $\Tilde{p} = \Tilde{\psi}^{-1}(\mathbf{x})$. Using the chain rule, we have
\begin{align*}
    J_{\ell \circ \Pi^{-1} \circ \Tilde{\psi}^{-1}}(\mathbf{x}) &= J_{\ell \circ \Pi^{-1}}(\Tilde{\psi}^{-1}(\mathbf{x})) J_{\Tilde{\psi}^{-1}}(\mathbf{x})\\
    &= \begin{bmatrix}
    I_{C-1} - \mathbbm{1}_{C-1}\Tilde{p}^{\top}\\
    - \frac{\Tilde{p}^{\top}}{p_C} (I_{C-1} - \mathbbm{1}_{C-1}\Tilde{p}^{\top})
    \end{bmatrix}
    H_{\Tilde{\underline{L}}}(\Tilde{p})
    \bigl(H_{-\Tilde{\underline{L}}}(\Tilde{p})\bigr)^{-1}\\
    &=-\begin{bmatrix}
    I_{C-1} - \mathbbm{1}_{C-1}\Tilde{p}^{\top}\\
    - \frac{\Tilde{p}^{\top}}{p_C} (I_{C-1} - \mathbbm{1}_{C-1}\Tilde{p}^{\top})
    \end{bmatrix}
\end{align*}
where the second equality comes as a result of Proposition \ref{proposition:bayes_risk_jacobian_hessian} and the inverse function theorem yielding 
\begin{align*}
    J_{\Tilde{\psi}^{-1}}(\mathbf{x}) &= \bigl(J_{\Tilde{\psi}}(\Tilde{p})\bigr)^{-1}\\
    &= \bigl(H_{-\Tilde{\underline{L}}}(\Tilde{p})\bigr)^{-1}.
\end{align*}

To prove our claim for all $i\in \{1,\dots,C\}$, we proceed by considering the following two cases.
\paragraph{Case 1: $i<C$.}
Fix $i<C$. Denote $e_i \in \R^{C-1}$ as $i$-th standard basis vector of $\R^{C-1}$. Then we have
\begin{align*}
    J_{\ell_i \circ \Pi^{-1} \circ \Tilde{\psi}^{-1}}(\mathbf{x}) &= (e_i, 0)^{\top} J_{\ell \circ \Pi^{-1} \circ \Tilde{\psi}^{-1}}(\mathbf{x})\\
    &= -(e_i^{\top} - \Tilde{p}^{\top})\\
    &= -(e_i - \Tilde{\psi}^{-1}(\mathbf{x}))^{\top}.
\end{align*}
Differentiating $\bigl(J_{\ell_i \circ \Pi^{-1} \circ \Tilde{\psi}^{-1}}(\mathbf{x})\bigr)^{\top}$ gives us
\begin{align*}
    H_{\ell_i \circ \Pi^{-1} \circ \Tilde{\psi}^{-1}}(\mathbf{x}) &= J_{\Tilde{\psi}^{-1}}(\mathbf{x})\\
    &=\bigl(H_{-\Tilde{\underline{L}}}(\Tilde{p})\bigr)^{-1}
\end{align*}
Recall that $-\Tilde{\underline{L}}$ is twice-differentiable and convex, and $H_{-\Tilde{\underline{L}}}(\Tilde{p})$ is positive definite for all $\Tilde{p} \in \Tilde{\Delta}^{C-1}$. It follows that $H_{\ell_i \circ \Pi^{-1} \circ \Tilde{\psi}^{-1}}(\mathbf{x})=\bigl(H_{-\Tilde{\underline{L}}}(\Tilde{p})\bigr)^{-1}$ exists and is positive definite for all $\Tilde{p} \in \Tilde{\Delta}^{C-1}$. Thus, $\ell_i \circ \Pi^{-1} \circ \Tilde{\psi}^{-1}(\mathbf{x})$ is a strictly convex function. This holds for arbitrary $i<C$.

\paragraph{Case 2: $i=C$.}
We have 
\begin{align*}
    J_{\ell_C \circ \Pi^{-1} \circ \Tilde{\psi}^{-1}}(\mathbf{x}) &= \frac{\Tilde{p}^{\top}}{p_C} (I_{C-1} - \mathbbm{1}_{C-1}\Tilde{p}^{\top})\\
    &= \frac{1}{p_C} (\Tilde{p}^{\top} - (1-p_C)\Tilde{p}^{\top})\\
    &=\Tilde{p}^{\top}\\
    &=\bigl(\Tilde{\psi}^{-1}(\mathbf{x})\bigr)^{\top}
\end{align*}
Differentiating $\bigl(J_{\ell_C \circ \Pi^{-1} \circ \Tilde{\psi}^{-1}}(\mathbf{x})\bigr)^{\top}$, we have 
\begin{align*}
    H_{\ell_C \circ \Pi^{-1} \circ \Tilde{\psi}^{-1}}(\mathbf{x}) &= J_{\Tilde{\psi}^{-1}}(\mathbf{x})\\
    &=\bigl(H_{-\Tilde{\underline{L}}}(\Tilde{p})\bigr)^{-1}
\end{align*}
Following the same argument as Case 1 shows that $\ell_C \circ \Pi^{-1} \circ \Tilde{\psi}^{-1}$ is a strictly convex function.
\end{proof}

Theorem \ref{theorem:proper_canonical_losses_convex} shows that any differentiable proper loss can be transformed such that its partial losses are convex functions, by equipping the loss with its associated canonical link. This illustrates the importance of pairing a proper loss with its corresponding canonical link. To complete a pairing of proper loss and canonical link, it is sufficient to learn one of these functions.

While the machinery of behind proper losses and canonical links have demonstrated the attractiveness of working with proper canonical losses, we have yet to define an analytical form of the loss we aim to learn. We conclude this chapter by providing the expression of a proper canonical loss and the proper loss (up to a projection), given a known canonical link. 
\begin{theorem}
\label{theorem:proper_loss_expression}
Let $\Tilde{\underline{L}}: \Tilde{\Delta}^{C-1} \to \R$ be a twice-differentiable projected conditional Bayes risk with $H_{-\Tilde{\underline{L}}}(\Tilde{p})$ being positive definite for all $\Tilde{p} \in \Tilde{\Delta}^{C-1}$ and $\Tilde{\psi}: \Tilde{\Delta}^{C-1} \to \R^{C-1}$ be its associated canonical link. Then we have
\begin{align*}
    (\ell \circ \Pi^{-1} \circ \Tilde{\psi}^{-1})(\mathbf{x}) &=
    \begin{bmatrix}
    \bigl((-\Tilde{\underline{L}})^{*}(\mathbf{x})\bigr)\mathbbm{1}_{C-1} - \mathbf{x}\\
    (-\Tilde{\underline{L}})^{*}(\mathbf{x})
    \end{bmatrix}
\end{align*}
where $(-\Tilde{\underline{L}})^{*}$ is the Legendre-Fenchel conjugate of $-\Tilde{\underline{L}}$. Moreover, we also have
\begin{align*}
    (\ell \circ \Pi^{-1})(\Tilde{p}) &=
    \begin{bmatrix}
    \left(\bigl((-\Tilde{\underline{L}})^{*} \circ \Tilde{\psi}\bigr) (\Tilde{p})\right)\mathbbm{1}_{C-1} - \Tilde{\psi}(\Tilde{p})\\
    \bigl((-\Tilde{\underline{L}})^{*} \circ \Tilde{\psi}\bigr) (\Tilde{p})
    \end{bmatrix}
\end{align*}
\end{theorem}
\begin{proof}
Following the workings of Theorem \ref{theorem:proper_canonical_losses_convex}, we have 
\begin{align*}
    J_{\ell_i \circ \Pi^{-1} \circ \Tilde{\psi}^{-1}}(\mathbf{x})&= -(e_i - \Tilde{\psi}^{-1}(\mathbf{x}))^{\top} \text{ for $i<C$}
\end{align*}
and 
\begin{align*}
    J_{\ell_C \circ \Pi^{-1} \circ \Tilde{\psi}^{-1}}(\mathbf{x}) 
    &=\bigl(\Tilde{\psi}^{-1}(\mathbf{x})\bigr)^{\top}.
\end{align*}
The properties of Legendre functions give us $\Tilde{\psi}^{-1} = \nabla (-\Tilde{\underline{L}})^{*}$. This allows us to deduce that, up to an additive constant, we have
\begin{align*}
    (\ell_i \circ \Pi^{-1} \circ \Tilde{\psi}^{-1})(\mathbf{x}) = (-\Tilde{\underline{L}})^{*}(\mathbf{x}) - e_i^{\top}\mathbf{x} \text{ for $i<C$}
\end{align*}
and
\begin{align*}
    (\ell_C \circ \Pi^{-1} \circ \Tilde{\psi}^{-1})(\mathbf{x}) 
    &=(-\Tilde{\underline{L}})^{*}(\mathbf{x}).
\end{align*}
We can rewrite this in matrix form as
\begin{align*}
    (\ell \circ \Pi^{-1} \circ \Tilde{\psi}^{-1})(\mathbf{x}) &=
    \begin{bmatrix}
    \bigl((-\Tilde{\underline{L}})^{*}(\mathbf{x})\bigr)\mathbbm{1}_{C-1} - \mathbf{x}\\
    (-\Tilde{\underline{L}})^{*}(\mathbf{x})
    \end{bmatrix}.
\end{align*}
Substituting $\Tilde{\psi}(\Tilde{p}) = \mathbf{x}$ in the above gives us the expression for $(\ell \circ \Pi^{-1})(\Tilde{p})$. This completes the proof.
\end{proof}
Theorem \ref{theorem:proper_loss_expression} provides us with an analytical expression for the proper canonical loss. It also shows that proper canonical losses are formulated by using the Legendre-Fenchel conjugate $(-\Tilde{\underline{L}})^{*} (\mathbf{x})$ as a baseline for the $C$-th partial loss at a point $\mathbf{x}$ and that all other partial losses are calculated by using $\mathbf{x} = \Tilde{\psi}(\Tilde{p})$ as an offset. This is intuitively sensible for two reasons. First, the partial loss for one class is only significant when compared against the partial losses of other classes. Second, we generally desire partial losses that are convex and proper canonical losses inherit this property from the convexity of $(-\Tilde{\underline{L}})^{*}$.

\section{Redefining Multinomial Logistic Regression}
\label{appendix:redefining_multinomial_logistic}
In this section, we first refine the definition of the categorical distribution by introducing the \emph{projected categorical distribution} as the natural multiclass analogue of the Bernoulli distribution. We then provide a principled reformulation of multinomial logistic regression in the framework of generalised linear models; by providing a canonical link function.

\subsection{Projected Categorical Distribution}
To present the definition of the projected categorical distribution, we first revisit the definition of the Exponential Family of probability distributions.
\begin{definition}[Exponential Family]
\label{definition:exponential_family}
A probability distribution belongs to an exponential family of distributions if its probability density has the form
\begin{align*}
    f(x) = h(x)\exp\bigl(\pmb{\theta}^{\top}\phi(x) - A(\pmb{\theta}) \bigr)
\end{align*}
where $\pmb{\theta} \in \R^n$ are the natural parameters, $\phi(x) \in \R^n$ is the vector of sufficient statistics, $A(\pmb{\theta}) \in \R$ is the log-partition function and $h(x) \in \R$ is the base measure. Members of the exponential family where there are no linear constraints on $\pmb{\theta}$ nor $\phi(x)$ are termed minimal or to have minimal form.
\end{definition}

With Definition \ref{definition:exponential_family} in hand, we can now formulate the categorical distribution in minimal form.
\begin{proposition}[Categorical Distribution]
\label{proposition:categorical}
A random vector $\mathbf{x}\in \{0,1\}^C$ with $\sum_{k=1}^C x_k = 1$, has a categorical distribution with $C$ categories if it has probability density in minimal form given by
\begin{align*}
    f(\mathbf{x}) = \exp\left(\sum_{k=1}^{C-1}[\![x_k = 1]\!]\theta_k - \log\left(1 + \sum_{k=1}^{C-1} \exp(\theta_k)\right)\right)
\end{align*}
where $\pmb{\theta} \in \R^{C-1}$, $\phi(\mathbf{x})= ([\![x_1 = 1]\!], \dots , [\![x_{C-1} = 1]\!])^{\top}$, $A(\pmb{\theta})=\log\left(1 + \sum_{k=1}^{C-1} \exp(\theta_k)\right)$ and $h(\mathbf{x}) = 1$. We denote the distribution of $\mathbf{x}$ as $\mathbf{x} \sim \text{Categorical}(p)$ with probability parameters $p = (\Pi^{-1} \circ \nabla A)(\pmb{\theta}) \in \Delta^{C-1}$.
\end{proposition}
\begin{proof}
Let $\mathbf{x}\in \{0,1\}^C$ with $\sum_{k=1}^C x_k = 1$ be a random vector that has a categorical distribution with probabilities $p \in \Delta^{C-1}$. We can rewrite the probability density function of the categorical distribution as
\begin{align*}
    f(\mathbf{x}) &= \exp\left(\sum_{i=1}^C [\![x_i = 1]\!] \log(p_i)\right)\\
    &= \exp\left(\sum_{i=1}^{C-1} [\![x_i = 1]\!] \log(p_i) + [\![x_C = 1]\!] \log(p_C)\right)\\
    &= \exp\left(\sum_{i=1}^{C-1} [\![x_i = 1]\!] \log(p_i) + \left(1-\sum_{i=1}^{C-1} [\![x_i = 1]\!]\right) \log\left(1-\sum_{i=1}^{C-1}p_i\right)\right)\\
    &= \exp\left(\sum_{i=1}^{C-1} [\![x_i = 1]\!] \log\left(\frac{p_i}{1-\sum_{j=1}^{C-1}p_j}\right) + \log\left(1-\sum_{i=1}^{C-1}p_i\right)\right)\\
    &= \exp\left(\sum_{i=1}^{C-1} [\![x_i = 1]\!] \theta_i - \log\left(1 + \sum_{k=1}^{C-1} \exp(\theta_k)\right)\right)
\end{align*}
where we let $\theta_i = \log\left(\frac{p_i}{1-\sum_{j=1}^{C-1}p_j}\right)$ for each $i=1,\dots,C-1$. We note the third equality results from the constraints that $\sum_{i=1}^{C} [\![x_i = 1]\!] =1$ and $\sum_{i=1}^{C} p_i = 1$, and the last equality comes from the observation that 
\begin{align*}
    \frac{1}{1-\sum_{i=1}^{C-1}p_i} = 1 + \frac{\sum_{i=1}^{C-1}p_i}{1-\sum_{i=1}^{C-1}p_i}.
\end{align*}
\end{proof}

While the formulation of categorical distribution in Proposition \ref{proposition:categorical} is not conventional, we note that its minimal form is reasonable as we only require sufficient statistics and probabilities of the first $C-1$ classes to specify the probability density. However, the specification of a categorically distributed variable can be further simplified due to the constraint that $\sum_{k=1}^C x_k = 1$. Concretely, all the randomness of a categorical random variable is fully captured in the first $C-1$ components. We now present the projected categorical distribution by exploiting a simplified specification of categorical distribution from Proposition \ref{proposition:categorical} and show that it reduces to the Bernoulli distribution when $C=2$.
\begin{definition}[Projected Categorical Distribution]
\label{definition:projected_categorical}
A random vector $\Tilde{\mathbf{x}}\in \{0,1\}^{C-1}$ with $\sum_{k=1}^C \Tilde{x}_k \leq 1$, has a projected categorical distribution with $C-1$ categories if it has probability density in minimal form given by
\begin{align*}
    f(\Tilde{\mathbf{x}}) = \exp\left(\sum_{k=1}^{C-1}[\![\Tilde{x}_k = 1]\!]\theta_k - \log\left(1 + \sum_{k=1}^{C-1} \exp(\theta_k)\right)\right)
\end{align*}
where $\pmb{\theta} \in \R^{C-1}$, $\phi(\Tilde{\mathbf{x}})= ([\![\Tilde{x}_1 = 1]\!], \dots , [\![\Tilde{x}_{C-1} = 1]\!])^{\top}$, $A(\pmb{\theta})=\log\left(1 + \sum_{k=1}^{C-1} \exp(\theta_k)\right)$ and $h(\mathbf{x}) = 1$.  We denote the distribution of $\Tilde{\mathbf{x}}$ as $\mathbf{x} \sim \text{projCategorical}(\Tilde{p})$ with probability parameters $\Tilde{p} = \nabla A(\pmb{\theta}) \in \Tilde{\Delta}^{C-1}$.
\end{definition}
\begin{corollary}
\label{corollary:projected_categorical_simplifies_to_bernoulli}
Suppose $\Tilde{x}\in \{0,1\}$ has a projected categorical distribution. Then $\Tilde{x}$ has a Bernoulli distribution.
\end{corollary}
\begin{proof}
We can write the probability density of $\Tilde{x}$ as
\begin{align*}
    f(\Tilde{x}) &= \exp\left([\![\Tilde{x} = 1]\!]\theta_k - \log\bigl(1 + \exp(\theta_k)\bigr)\right)\\
    &= \frac{\left(\frac{p}{1-p}\right)^{[\![\Tilde{x} = 1]\!]}}{1 + \frac{p}{1-p}}\\
    &= p^{[\![\Tilde{x} = 1]\!]} (1-p)^{1-[\![\Tilde{x} = 1]\!]}.
\end{align*}
Hence, $\Tilde{x}$ has support on $\{0,1\}$ with density $f(\Tilde{x})$ which matches the density of a Bernoulli distribution. It follows that $\Tilde{x}$ has a Bernoulli distribution.
\end{proof}
We note that although the probability density of the categorical distribution in Proposition \ref{proposition:categorical} can simplify to the probability density of the Bernoulli distribution when $C=2$, the support of the resultant categorical distribution is $\{0,1\}^2$ which differs from the support of the Bernoulli distribution given by $\{0,1\}$. Corollary \ref{corollary:projected_categorical_simplifies_to_bernoulli} illustrates that the projected categorical distribution is the natural multiclass analogue of the Bernoulli distribution.

We conclude this section by noting that we can sample from the categorical distribution by transforming a sample from the projected categorical distribution, and vice versa.
\begin{corollary}
If $\Tilde{\mathbf{x}}\in \{0,1\}^{C-1}$ has a projected categorical distribution with parameters $\Tilde{p} \in \Tilde{\Delta}^{C-1}$, then
\begin{align*}
    \mathbf{x} &= \Pi^{-1}(\Tilde{\mathbf{x}})\\
    &=
    \begin{bmatrix}
    I_{C-1}\\
    - \mathbbm{1}_{C-1}^{\top}
    \end{bmatrix}
    \Tilde{\mathbf{x}} + 
    \begin{bmatrix}
    \mathbf{0}_{C-1}\\
    1
    \end{bmatrix}
\end{align*}
has a categorical distribution with parameters $p = \Pi^{-1}(\Tilde{p}) \in \Delta^{C-1}$, where $\mathbf{0}_{C-1} \in \R^{C-1}$ denotes a vector of zeroes. Similarly if $\mathbf{x}\in \{0,1\}^{C-1}$ has a categorical distribution with parameters $p \in \Delta^{C-1}$, then 
\begin{align*}
    \Tilde{\mathbf{x}} &= \Pi(\mathbf{x})\\
    &=
    \begin{bmatrix}
    I_{C-1} \quad \mathbf{0}_{C-1}
    \end{bmatrix}
    \mathbf{x}
\end{align*}
has a projected categorical distribution with parameters $\Tilde{p} = \Pi(p) \in \Tilde{\Delta}^{C-1}$.
\end{corollary}
\begin{proof}
This follows from the definitions of the projection map $\Pi$ and its inverse $\Pi^{-1}$, and the constraints on the elements of $\Tilde{\mathbf{x}}$ and $\mathbf{x}$.
\end{proof}

\subsection{Multinomial Logistic Regression as a Generalised Linear Model}
\label{section:multinomial_logistic_as_glm}
To facilitate our discussion in this section on generalised linear models, we first present a refined formulation of generalised linear models pioneered by \citet{nelder1972generalized}.

\begin{definition}[Generalised Linear Model]
\label{definition:glm}
Let $\mathbf{x} \in \R^p$ be a set of independent variables and $y \in \R$ be a dependent variable. A generalised linear model of $(\mathbf{x}, y)$ consists of the following assumptions:
\begin{itemize}
    \item the probability distribution of $y$, denoted $p_y$, belongs to the Exponential family with natural parameters $\pmb{\theta} \in \R^{C-1}$ and log-partition function $A(\pmb{\theta})$

    \item $\pmb{\theta} =\mathbf{W}^{\top} \mathbf{x} + \mathbf{b}$ where $\mathbf{W}\in \R^{(C-1)\times p}$ and $\mathbf{b}\in \R^{C-1}$;

    \item There exists a smooth and strictly monotone canonical link function $\Tilde{\psi}$ such that $\Tilde{\psi}^{-1}(\mathbf{W}^{\top} \mathbf{x} + \mathbf{b}) = \pmb{\mu}$ where $\pmb{\mu}$ is the mean parameter of $p_y$.
\end{itemize}
\end{definition}
\begin{lemma}
\label{lemma:log_partition_derivatives}
Suppose $(\mathbf{x}, y)$ follow a generalised linear model where $p_y$ has natural parameters $\pmb{\theta} \in \R^{C-1}$ and log-partition function $A(\pmb{\theta})$.
Then the natural parameters $\pmb{\theta}$ relate to the mean parameters $\pmb{\mu}$ and variance parameters $\pmb{\Sigma}$ of $p_y$ through the following equalities
\begin{align*}
    \pmb{\mu} &= \nabla A(\pmb{\theta}),\\
    \pmb{\Sigma} &= \nabla^2 A(\pmb{\theta}).
\end{align*}
\end{lemma}

Recall that \emph{univariate} generalised linear models assume the following model for the natural parameters of a probability distribution that belongs to the Exponential family as
\begin{align*}
    \theta &= \mathbf{w}^{\top} \mathbf{x} + b\\
    &= \Tilde{\psi}(\mu)
\end{align*}
where $\theta \in \R$ is the natural parameter, $\mu \in \R$ is the mean parameter, $\Tilde{\psi}$ is the canonical link, $\mathbf{w} \in \R^p$ is the vector of coefficients and $b\in \R$ is the intercept. The usage of generalised linear models to model \emph{scalar-valued} responses is well-known. Examples include Poisson regression for count data and logistic regression for binary outcomes. As responses are univariate, the requirements of the link function $\psi$ reduce to being invertible and strictly increasing. The latter property serves as the foundation of interpretability of the effects of covariates on the response by noting the sign of coefficients. The above formulation suffices for the modelling of responses as they are often univariate and can be accordingly described by an appropriate univariate probability distribution.

Extending generalised linear models to multiclass problems is not straightforward as the response is now multivariate as each label $y_n \in \{1,\dots,C\}$ is often represented as a standard basis vector $e_{y_n} \in \R^{C}$. Multiclass probability estimates $p \in \Delta^{C-1} \subset \R^{C}$ are then formed to approximate $e_{y_n}$. To pose multinomial logistic regression as a generalised linear model, we require a canonical link. That is, we must define an invertible multivariate function with a multivariate image, and equipped with a property analogous to the \emph{strictly increasing} property for the univariate case. The latter refers to an order-preserving property of the link which can be done in $\R$ but not in $\R^{C}$ for general $C>1$. The theory of monotone operators overcomes this difficulty. Specifically, strict monotonicity from Definition \ref{definition:monotone} subsumes the idea of strictly increasing maps and is equipped with a more general definition. A strictly monotone map is also invertible by Corollary \ref{corollary:strictly_monotone_invertible}. This makes strict monotonicity more readily applicable to multivariate functions with a multivariate image and justifying our refined definition of generalised linear models in Definition \ref{definition:glm}.

We note that the conventional formulation of multinomial logistic regression utilises the softmax function as the \emph{inverse} link that maps to probabilities. The softmax function is defined as 
\begin{align*}
    u:\R^{C} \to \Delta^{C-1},\quad
    u(\mathbf{x}) = \left(\frac{\exp(x_i)}{\sum_{k=1}^{C} \exp(x_k)}\right)_{1\leq i \leq C}.
\end{align*}
However, the softmax function does not correspond to a valid canonical link function as it is not invertible. To observe this, we note that $\mathbf{x} = (x_1,\dots,x_C)$ and $\mathbf{z} = (x_1 +z,\dots,x_C+z)$ would yield the same set of probabilities $p=\left(\frac{\exp(x_i)}{\sum_{k=1}^{C} \exp(x_k)}\right)_{1\leq i \leq C}$. In other words, the pre-image of $p$ is not unique. This implies that the conventional formulation of multinomial logistic regression is not a generalised linear model as the last assumption of Definition \ref{definition:glm} is not satisfied. 

In the remainder of this section, we seek to formalise multinomial logistic regression as a generalised linear model by stating a valid canonical link function. To determine a valid canonical link for multinomial logistic regression, we now present a function that is equipped with invertibility, and later show it is the \emph{inverse} canonical link for the projected categorical distribution. 

\begin{corollary}
\label{corollary:softmax+_inverse}
Let $u$ be the $\text{softmax}^+$ function defined as
\begin{align*}
    u:\R^{C-1} \to \Tilde{\Delta}^{C-1},\quad
    u(\mathbf{x}) = \left(\frac{\exp(x_i)}{1+ \sum_{k=1}^{C-1} \exp(x_k)}\right)_{1\leq i \leq C-1},
\end{align*}
and $g$ be defined as 
\begin{align*}
    g:\Tilde{\Delta}^{C-1} \to \R^{C-1},\quad
    g(\Tilde{p}) = \left(\log\left(\frac{\Tilde{p}_i}{1-\sum_{k=1}^{C-1}\Tilde{p}_k}\right)\right)_{1\leq i \leq C-1}.
\end{align*}
Then $g$ is the inverse function of $u$.
\end{corollary}
\begin{proof}
We must show that $g \circ u$ and $u\circ g$ are both identity functions. We have
\begin{align*}
    (g \circ u)(\mathbf{x}) &= \left(\log\left(\frac{\frac{\exp(x_i)}{1+ \sum_{k=1}^{C-1} \exp(x_k)}}{1 - \sum_{j=1}^{C-1}\frac{\exp(x_j)}{1+ \sum_{k=1}^{C-1} \exp(x_k)}}\right)\right)_{1\leq i \leq C-1}\\
    &= \left(\log\bigl(\exp(x_i)\bigr)\right)_{1\leq i \leq C-1}\\
    &= (x_i)_{1\leq i \leq C-1},
\end{align*}
and 
\begin{align*}
    (u \circ g)(\Tilde{p}) &= \left(\frac{\frac{\Tilde{p}_i}{1-\sum_{i=1}^{C-1}\Tilde{p}_i}}{1+ \sum_{k=1}^{C-1} \frac{\Tilde{p}_k}{1-\sum_{i=1}^{C-1}\Tilde{p}_i}}\right)_{1\leq i \leq C-1}\\
    &= \left(\frac{\Tilde{p}_i}{1-\sum_{k=1}^{C-1}\Tilde{p}_i+\sum_{k=1}^{C-1}\Tilde{p}_i}\right)_{1\leq i \leq C-1}\\
    &= (\Tilde{p}_i)_{1\leq i \leq C-1}.
\end{align*}
Thus, $g$ is the inverse of $u$.
\end{proof}

\begin{corollary}
\label{corollary:inverse_softmax+_strictly_monotone}
Let $g$ be the inverse of the $\text{softmax}^+$ function defined as
\begin{align*}
    g:\Tilde{\Delta}^{C-1} \to \R^{C-1},\quad
    g(\Tilde{p}) = \left(\log\left(\frac{\Tilde{p}_i}{1-\sum_{k=1}^{C-1}\Tilde{p}_k}\right)\right)_{1\leq i \leq C-1}.
\end{align*}
Then $g$ is smooth and strictly monotone.
\end{corollary}
\begin{proof}
It is clear that $g$ is smooth so we are left to prove it is strictly monotone.

Fix $\Tilde{p} \in \Tilde{\Delta}^{C-1}$. For ease of notation, we denote $M$ as $J_g(\Tilde{p})$ where $M_{ij}$ refers to the entry within the $i$-th row and $j$-th column of $J_g(\Tilde{p})$. Consider any row $i\in \{1,\dots,C-1\}$. We have
\begin{align*}
    M_{ii} &= \frac{1}{\Tilde{p}_i} + \frac{1}{1 - \sum_{k=1}^{C-1}\Tilde{p}_k},\\
    M_{ij} &= \frac{1}{1 - \sum_{k=1}^{C-1}\Tilde{p}_k}.
\end{align*}
That is, $M=D +\frac{1}{1 - \sum_{k=1}^{C-1}\Tilde{p}_k}\mathbbm{1}_{C-1} \mathbbm{1}_{C-1}^{\top}$ where $D \in \R^{(C-1) \times (C-1)}$ is a diagonal matrix with entries $\frac{1}{\Tilde{p}_1},\dots,\frac{1}{\Tilde{p}_{C-1}}$. For any $\mathbf{z} \in \R^{C-1}$, we note that 
\begin{align*}
    \mathbf{z}^{\top} M \mathbf{z} &= \mathbf{z}^{\top}D\mathbf{z} +\frac{1}{1 - \sum_{k=1}^{C-1}\Tilde{p}_k}\mathbf{z}^{\top}\mathbbm{1}_{C-1} \mathbbm{1}_{C-1}^{\top}\mathbf{z}\\
    &= \mathbf{z}^{\top}D^{\frac{1}{2}}D^{\frac{1}{2}}\mathbf{z} +\frac{1}{1 - \sum_{k=1}^{C-1}\Tilde{p}_k}\mathbf{z}^{\top}\mathbbm{1}_{C-1} \mathbbm{1}_{C-1}^{\top}\mathbf{z}\\
    &= \|D^{\frac{1}{2}}\mathbf{z}\|_2^2 + \frac{1}{1 - \sum_{k=1}^{C-1}\Tilde{p}_k}\|\mathbbm{1}_{C-1}^{\top}\mathbf{z}\|_2^2\\
    &\geq 0
\end{align*}
where $\|\cdot\|_2$ is the Euclidean norm. Hence, $M$ is positive semi-definite so its eigenvalues are non-negative. Using the matrix determinant lemma, we have
\begin{align*}
    |M| &= \left(1 + \frac{1}{1 - \sum_{k=1}^{C-1}\Tilde{p}_k}\mathbbm{1}_{C-1}^{\top} D^{-1}\mathbbm{1}_{C-1} \right) |D|\\
    &=\left(1 + \frac{\sum_{k=1}^{C-1}\Tilde{p}_k}{1 - \sum_{k=1}^{C-1}\Tilde{p}_k} \right) \frac{1}{\prod_{k=1}^{C-1}\Tilde{p}_k}\\
    &> 0.
\end{align*}
This implies that all eigenvalues of $M$ are positive. Thus, $M$ is positive definite so it follows that $g$ is strictly monotone by Theorem \ref{theorem:monotone_iff_jacobian_psd}.
\end{proof}

Corollary \ref{corollary:inverse_softmax+_strictly_monotone} deduces that the inverse of the $\text{softmax}^+$ function meets the requirements of smoothness and strict monotonicity that we seek in a canonical link.

 With Proposition \ref{definition:projected_categorical} in hand, we can now deduce that the inverse of the $\text{softmax}^+$ function is the canonical link for the projected categorical distribution.
\begin{theorem}
\label{theorem:categorical_canonical_link}
Let $u$ be the $\text{softmax}^+$ function defined as
\begin{align*}
    u:\R^{C-1} \to \Tilde{\Delta}^{C-1},\quad
    u(\mathbf{x}) = \left(\frac{\exp(x_i)}{1+ \sum_{k=1}^{C-1} \exp(x_k)}\right)_{1\leq i \leq C-1},
\end{align*}
with inverse 
\begin{align*}
    g:\Tilde{\Delta}^{C-1} \to \R^{C-1},\quad
    g(\Tilde{p}) = \left(\log\left(\frac{\Tilde{p}_i}{1-\sum_{k=1}^{C-1}\Tilde{p}_k}\right)\right)_{1\leq i \leq C-1}.
\end{align*}
Then $g$ and $u$ are the respective canonical link and inverse canonical link corresponding to the projected categorical distribution.
\end{theorem}
\begin{proof}
We first note that $\Tilde{\psi}^{-1}(\pmb{\theta}) = \pmb{\mu} = \nabla A(\pmb{\theta})$ from Definition \ref{definition:glm} and Lemma \ref{lemma:log_partition_derivatives}. From properties of Theorem \ref{theorem:legendretron_is_legendre} and Definition \ref{definition:projected_categorical}, we note that $A(\pmb{\theta})$ is $\text{LogSumExp}^+$, so it follows that $u$, the $\text{softmax}^+$ function, is the inverse canonical link corresponding to the projected categorical distribution. Hence, $g$ is the canonical link corresponding to the projected categorical distribution.
\end{proof}
\begin{corollary}
Let $\Tilde{\mathbf{x}}$ be a projected categorical random variable with $C-1$ categories with natural parameters $\pmb{\theta}$ and log-partition function $A(\pmb{\theta})$. Then its covariance parameters $\Tilde{\pmb{\Sigma}}$ relate to the mean parameters $\Tilde{\pmb{\mu}}$ through the following equality
\begin{align*}
    \Tilde{\pmb{\Sigma}} &= \nabla^2 A(\pmb{\theta}) \\
    &= D_{\Tilde{\pmb{\mu}}} - \Tilde{\pmb{\mu}}\Tilde{\pmb{\mu}}^{\top}.
\end{align*}
where $D_{\Tilde{\pmb{\mu}}}$ denotes a $(C-1) \times (C-1)$ diagonal matrix with entries given by $\Tilde{\pmb{\mu}}$.
\end{corollary}
\begin{proof}
Let $\Tilde{p} = u(\pmb{\theta})$. Then $\Tilde{p} = \Tilde{\pmb{\mu}}$ and we can express the Hessian from Theorem \ref{theorem:legendretron_is_legendre} as 
\begin{align*}
    J_u(\pmb{\theta}) &= \nabla^2 A(\pmb{\theta})\\
    &= D_{\Tilde{p}} - \Tilde{p}\Tilde{p}^{\top}\\
    &=D_{\Tilde{\pmb{\mu}}} - \Tilde{\pmb{\mu}}\Tilde{\pmb{\mu}}^{\top}
\end{align*}
where $D_{\Tilde{p}}$ and $D_{\Tilde{\pmb{\mu}}}$ denote $(C-1) \times (C-1)$ diagonal matrices with entries given by $\Tilde{p}$ and $\Tilde{\pmb{\mu}}$ respectively. Elements of the covariance matrix $\Tilde{\pmb{\Sigma}}$ are given by
\begin{align*}
    \Tilde{\Sigma}_{ij} &= \mathbb{E}[(\Tilde{x}_i - \Tilde{\mu}_i)(\Tilde{x}_j - \Tilde{\mu}_j)]\\
    &=\mathbb{E}[\Tilde{x}_i\Tilde{x}_j] - \Tilde{\mu}_i\Tilde{\mu}_j\\
    &=
    \begin{cases}
        \Tilde{\mu}_i - \Tilde{\mu}_i\Tilde{\mu}_i  & \text{ when $i=j$} \\
        - \Tilde{\mu}_i\Tilde{\mu}_j & \text{ otherwise}
    \end{cases}
\end{align*}
where the last equality comes from the constraint $\sum_{k=1}^{C-1}\Tilde{x}_k \leq 1$ and the fact that $\Tilde{\mathbf{x}} \in \{0,1\}^{C-1}$. We can rewrite the above as $\Tilde{\pmb{\Sigma}} = D_{\Tilde{\pmb{\mu}}} - \Tilde{\pmb{\mu}}\Tilde{\pmb{\mu}}^{\top}$. Thus, we can deduce 
\begin{align*}
    \Tilde{\pmb{\Sigma}} &= \nabla^2 A(\pmb{\theta}) \\
    &= D_{\Tilde{\pmb{\mu}}} - \Tilde{\pmb{\mu}}\Tilde{\pmb{\mu}}^{\top}.
\end{align*}
\end{proof}

 With the canonical link of the projected categorical distribution known, we can now express multinomial logistic regression as a generalised linear model.
\begin{algorithm}[tb]
   \caption{Multinomial Logistic Regression}
   \label{algorithm:multinomial_logistic}
\begin{algorithmic}
   \STATE {\bfseries Input:} sample $\mathcal{S}\subset\mathcal{D}$, number of iterations $T$, function $u=\text{softmax}^+$.
   \STATE Initialise $\mathbf{W}$ and $\mathbf{b}$.
   \FOR{$i=1$ {\bfseries to} $T$}
   \FOR{each $(\mathbf{x}_n, y_n) \in \mathcal{S}$}
   \STATE Compute $\mathbf{z}_n = \mathbf{W}\mathbf{x}_n + \mathbf{b}$.
   \STATE Compute $\hat{p}(\mathbf{z}_n)=u(\mathbf{z}_n)$.
   \ENDFOR
   \STATE Compute $\mathbb{E}_{\mathcal{S}}[\mathcal{L}((\Pi^{-1} \circ \hat{p})(\mathbf{z}), y)]$ by Monte Carlo where $\mathcal{L}$ is the log-likelihood of the Categorical distribution.
   \STATE Update $\mathbf{W}$ and $\mathbf{b}$ by gradient descent.
   \ENDFOR
   \STATE {\bfseries Output:} $\mathbf{W}$ and $\mathbf{b}$.
\end{algorithmic}
\end{algorithm}

Given a dataset $\mathcal{D}=\{(\mathbf{x}_n, y_n)\}_{n=1}^N$, we have the classification model
\begin{align*}
    \Tilde{y}_n|\mathbf{x}_n &\sim \text{projCategorical}\bigl(\hat{p}(\mathbf{z}_n)\bigr) \text{ where $\mathbf{z}_n = \mathbf{W}\mathbf{x}_n + \mathbf{b}$}
\end{align*}
or equivalently,
\begin{align*}
    y_n|\mathbf{x}_n &\sim \text{Categorical}\bigl((\Pi^{-1} \circ \hat{p})(\mathbf{z}_n)\bigr) \text{ where $\mathbf{z}_n = \mathbf{W}\mathbf{x}_n + \mathbf{b}$}
\end{align*}
where $\mathbf{W}\in \R^{(C-1)\times p}$, $\mathbf{b}\in \R^{C-1}$ and $\hat{p}(\mathbf{z}_n) = u(\mathbf{z}_n)$ with $u$ being the $\text{softmax}^+$ function. Algorithm \ref{algorithm:multinomial_logistic} describes multinomial logistic regression in detail.

\section{Proof of Theorem \ref{theorem:composition_gradient_iff_psd}}
\label{appendix:composition_iff_psd}
(1) $\implies$ (2). This follows from Schwarz's theorem on the equality of mixed partial derivatives. (2) $\implies$ (3). This follows from Theorem \ref{theorem:matrix_products_psd_iff_symmetric} since the Jacobian of a composite function is a product of matrices. (3) $\implies$ (4). This follows from Theorem \ref{theorem:monotone_iff_jacobian_psd}. We are left to prove (4) $\implies$ (1).

(4) $\implies$ (1). Since $f\circ g$ is monotone and continuous, it follows from Theorem \ref{theorem:continuous_monotone_implies_maximally_monotone} that $f\circ g$ is maximally monotone. From Theorem \ref{theorem:asplund_composition_maximal_monotone}, $(f \circ g)(\mathbf{x}) = \nabla F(\mathbf{x}) + L\mathbf{x}$ for a differentiable convex function $F$ and a skew-symmetric matrix $L$. Since $f \circ g$ is differentiable then it follows that $F$ is twice-differentiable. This gives us $J_{f \circ g} = \nabla^2 F + L$ where $\nabla^2 F$ is symmetric by Schwarz's theorem on the equality of mixed partial derivatives. Theorems \ref{theorem:monotone_iff_jacobian_psd} and \ref{theorem:matrix_products_psd_iff_symmetric} tell us that $J_{f \circ g}$ is also symmetric. As $J_{f \circ g}$ and $\nabla^2 F$ are both symmetric, then we must have $L = 0 = L^{\top}$. That is, $f \circ g = \nabla F$ where $F$ is twice-differentiable and convex.

\section{Proof of Theorem \ref{theorem:maps_with_pd_jacobians_closed_under_composition}}
\label{appendix:maps_with_pd_jacobians_closed_under_composition}
Fix $\mathbf{x} \in \R^{C-1}$. The Jacobian of $f \circ g$ is given by
\begin{align*}
    J_{f\circ g}(\mathbf{x}) = J_{f}(g(\mathbf{x})) J_{g}(\mathbf{x}).
\end{align*}
Here we aim to prove that $J_{f\circ g}(\mathbf{x})$ is positive definite. We first note that $J_{g}(\mathbf{x})$ is invertible since $|J_{g}(\mathbf{x})|>0$. Now, note that $J_{f\circ g}(\mathbf{x})$ is similar to the matrix 
\begin{align*}
    (J_{g}(\mathbf{x})) ^{\frac{1}{2}}J_{f}(g(\mathbf{x})) J_{g}(\mathbf{x})(J_{g}(\mathbf{x})) ^{-\frac{1}{2}} &= (J_{g}(\mathbf{x}))^{\frac{1}{2}}J_{f}(g(\mathbf{x})) (J_{g}(\mathbf{x}))^{\frac{1}{2}}\\
    &= (J_{g}(\mathbf{x}))^{\frac{1}{2}}(J_{f}(g(\mathbf{x})))^{\frac{1}{2}}(J_{f}(g(\mathbf{x})))^{\frac{1}{2}} (J_{g}(\mathbf{x}))^{\frac{1}{2}}
\end{align*}
where the square roots of the matrices $J_{f}(g(\mathbf{x}))$ and $J_{g}(\mathbf{x})$ are respectively given by $(J_{f}(g(\mathbf{x})))^{\frac{1}{2}}$ and $(J_{g}(\mathbf{x}))^{\frac{1}{2}}$ with both known to be symmetric and positive definite since $J_{f}(g(\mathbf{x}))$ and $J_{g}(\mathbf{x})$ are symmetric and positive definite. For any $\mathbf{z} \in \R^{C-1}$, we note that 
\begin{align*}
    \mathbf{z}^{\top}(J_{g}(\mathbf{x}))^{\frac{1}{2}}J_{f}(g(\mathbf{x})) (J_{g}(\mathbf{x}))^{\frac{1}{2}} \mathbf{z} &= \mathbf{z}^{\top}(J_{g}(\mathbf{x}))^{\frac{1}{2}}(J_{f}(g(\mathbf{x})))^{\frac{1}{2}}(J_{f}(g(\mathbf{x})))^{\frac{1}{2}} (J_{g}(\mathbf{x}))^{\frac{1}{2}} \mathbf{z}\\
    &=\|(J_{f}(g(\mathbf{x})))^{\frac{1}{2}} (J_{g}(\mathbf{x}))^{\frac{1}{2}} \mathbf{z}\|_2^2\\
    &\geq 0.
\end{align*}
We note the second equality follows from the symmetry of $(J_{f}(g(\mathbf{x})))^{\frac{1}{2}}$ and $(J_{g}(\mathbf{x}))^{\frac{1}{2}}$. It follows that $(J_{g}(\mathbf{x}))^{\frac{1}{2}}J_{f}(g(\mathbf{x})) (J_{g}(\mathbf{x}))^{\frac{1}{2}}$ is positive semi-definite and has non-negative eigenvalues.  As $J_{g}(\mathbf{x})$ and $J_{f}(g(\mathbf{x}))$ are positive definite, we also have 
\begin{align*}
    \left|(J_{g}(\mathbf{x}))^{\frac{1}{2}}J_{f}(g(\mathbf{x})) (J_{g}(\mathbf{x}))^{\frac{1}{2}}\right| = \left|J_{f}(g(\mathbf{x}))\right| \left|(J_{g}(\mathbf{x}))\right| > 0.
\end{align*}

It follows that all eigenvalues of $(J_{g}(\mathbf{x}))^{\frac{1}{2}}J_{f}(g(\mathbf{x})) (J_{g}(\mathbf{x}))^{\frac{1}{2}}$ must be positive so $(J_{g}(\mathbf{x}))^{\frac{1}{2}}J_{f}(g(\mathbf{x})) (J_{g}(\mathbf{x}))^{\frac{1}{2}}$ is positive definite. We can denote the eigenvalues $\lambda_1, \dots, \lambda_{C-1} \in \R$ such that $\lambda_1 \leq \lambda_2 \leq \dots \leq \lambda_{C-1}$. Since similar matrices have the same eigenvalues, it follows that $\lambda_1, \dots, \lambda_{C-1}$ are also the eigenvalues of $J_{f\circ g}(\mathbf{x})$. We note that $J_{f\circ g}(\mathbf{x})$ is not assumed to be symmetric so we cannot utilise Lemma \ref{lemma:definiteness_eigenvalues} to deduce positive definiteness of $J_{f\circ g}(\mathbf{x})$ here.

Denote $S = \frac{1}{2}(J_{f\circ g}(\mathbf{x}) + (J_{f\circ g}(\mathbf{x}) )^{\top})$ and $A = \frac{1}{2}(J_{f\circ g}(\mathbf{x}) - (J_{f\circ g}(\mathbf{x})) ^{\top})$ as the symmetric and skew-symmetric parts of $J_{f\circ g}(\mathbf{x})$ respectively. It is well known that for any skew-symmetric matrix $A$ and any $\mathbf{z} \in \R^{C-1}$, we have $\mathbf{z}^{\top}A \mathbf{z} = 0$. To prove that $J_{f\circ g}(\mathbf{x})$ is positive definite, it suffices to prove that $\mathbf{z}^{\top} J_{f\circ g}(\mathbf{x}) \mathbf{z}=\mathbf{z}^{\top}S \mathbf{z} > 0$ for any $\mathbf{z} \in \R^{C-1} \setminus \{\mathbf{0}\}$.

Firstly, recall that all eigenvalues of $J_{f\circ g}(\mathbf{x})$ are real and positive, and the fact that the transpose of $J_{f\circ g}(\mathbf{x})$, $(J_{f\circ g}(\mathbf{x}))^{\top}$, has the same eigenvalues as $J_{f\circ g}(\mathbf{x})$. That is, all eigenvalues of $(J_{f\circ g}(\mathbf{x}))^{\top}$ are real and positive. Secondly, $S = \frac{1}{2}(J_{f\circ g}(\mathbf{x}) + (J_{f\circ g}(\mathbf{x}) )^{\top})$ is symmetric so all of its eigenvalues must be real. Hence, Theorem \ref{theorem:eigenvalues_matrix_sum} gives us the following bound for the smallest eigenvalue $\lambda_1(S)$
\begin{align*}
    \lambda_1(S) &\geq \lambda_1\left(\frac{1}{2}J_{f\circ g}(\mathbf{x})\right) + \lambda_1\left(\frac{1}{2}(J_{f\circ g}(\mathbf{x}) )^{\top}\right)\\
    &=\frac{1}{2}\bigg(\lambda_1(J_{f\circ g}(\mathbf{x})) + \lambda_1((J_{f\circ g}(\mathbf{x}) )^{\top})\bigg)\\
    &>0.
\end{align*}

The Rayleigh quotient for $S$ and any $\mathbf{z} \in \R^{C-1} \setminus \{\mathbf{0}\}$, is given by $\frac{\mathbf{z}^{\top} S \mathbf{z}}{\|\mathbf{z}\|_2^2}$, and satisfies the inequality 
\begin{align*}
    \lambda_1(S) \leq \frac{\mathbf{z}^{\top} S \mathbf{z}}{\|\mathbf{z}\|_2^2} \leq \lambda_{C-1}(S).
\end{align*} 

Hence, we have
\begin{align*}
    \frac{\mathbf{z}^{\top} S \mathbf{z}}{\|\mathbf{z}\|_2^2} \geq \lambda_1(S) > 0 \text{ for all $\mathbf{z} \in \R^{C-1} \setminus \{\mathbf{0}\}$ }.
\end{align*}
Thus, $\mathbf{z}^{\top} J_{f\circ g}(\mathbf{x}) \mathbf{z} = \mathbf{z}^{\top} S \mathbf{z} >0, \forall\mathbf{z} \in \R^{C-1} \setminus \{\mathbf{0}\}$ and so, $J_{f\circ g}(\mathbf{x})$ is positive definite. This holds for arbitrary $\mathbf{x} \in \R^{C-1}$ so it follows that $f\circ g$ is the gradient of a twice-differentiable convex function $F$ by Theorem \ref{theorem:composition_gradient_iff_psd} with $F$ being strictly convex since $J_{f\circ g}(\mathbf{x})$ is positive definite. In other words, $f\circ g$ is the gradient of a twice-differentiable Legendre function.

\section{Proof of Theorem \ref{theorem:legendretron_is_legendre}}
\paragraph{Proof of Properties of LogSumExp$^+$ and softmax$^+$}
Since positive definiteness of $J_u(\mathbf{x})$ for all $\mathbf{x} \in \R^{C-1}$ implies strict convexity of $f$ and strict convexity of $f$ implies invertibility of $u$, it suffices to prove that $J_u(\mathbf{x})$ is positive definite for all $\mathbf{x} \in \R^{C-1}$.

Fix $\mathbf{x} \in \R^{C-1}$. For ease of notation, we denote $M$ as $J_u(\mathbf{x})$ where $M_{ij}$ refers to the entry within the $i$-th row and $j$-th column of $J_u(\mathbf{x})$. Consider any row $i\in \{1,\dots,C-1\}$. We have
\begin{align*}
    M_{ii} &= \frac{\exp(x_i)}{1+\sum_{k=1}^{C-1} \exp(x_k)}\left(1 - \frac{\exp(x_i)}{1+\sum_{k=1}^{C-1} \exp(x_k)}\right),\\
    M_{ij} &= -\frac{\exp(x_i)}{1+\sum_{k=1}^{C-1} \exp(x_k)}\frac{\exp(x_j)}{1+\sum_{k=1}^{C-1} \exp(x_k)}.
\end{align*}
Observe that $M_{ii} - \sum_{j\neq i}|M_{ij}| = \frac{\exp(x_i)}{1+\sum_{k=1}^{C-1} \exp(x_k)}\left(1- \frac{\sum_{k=1}^{C-1}\exp(x_k)}{1+\sum_{k=1}^{C-1} \exp(x_k)} \right) >0$. This holds for arbitrary $i\in \{1,\dots,C-1\}$ so it follows that $J_u(\mathbf{x})$ is strictly diagonally dominant. This implies that $J_u(\mathbf{x})$ is positive definite so it follows that $f$ is strictly convex. This completes the proof of the key properties of the $\text{LogSumExp}^+$ function and its gradient $\text{softmax}^+$.

\paragraph{Proof of functions learned by \textsc{LegendreTron} are inverse canonical links}
We first note that $v^{-1} = (\nabla g_1) \circ (\nabla g_2) \circ \dots \circ (\nabla g_B)$ is indeed invertible since the RHS is invertible by the strong convexity of $g_1, g_2, \dots, g_B$. Since each $\nabla g_i$ is symmetric and positive definite, it follows that $v^{-1}$ is the gradient of a twice-differentiable Legendre function by applying Theorem \ref{theorem:maps_with_pd_jacobians_closed_under_composition} recursively. It follows from Theorem \ref{theorem:composition_gradient_iff_psd} that $J_{v^{-1}}(\mathbf{x})$ is symmetric and positive semi-definite for all $\mathbf{x} \in \R^{C-1}$. Due to the strong convexity of each $g_i$, we also have that $|J_{v^{-1}}(\mathbf{x})| > 0$ so $J_{v^{-1}}(\mathbf{x})$ is positive definite for all $\mathbf{x} \in \R^{C-1}$.

Recall that $\text{LogSumExp}^+$ is twice-differentiable with gradient  $u=\text{softmax}^+$ and Hessian $J_u(\mathbf{x})$ being strictly diagonally dominant. That is, $J_u(\mathbf{x})$ is symmetric and positive definite for all $\mathbf{x} \in \R^{C-1}$. Applying Theorem \ref{theorem:maps_with_pd_jacobians_closed_under_composition} on $u\circ v^{-1}$ allows us to deduce that $u\circ v^{-1}$ is the gradient of a twice-differentiable Legendre function that maps to $\Tilde{\Delta}^{C-1}$ so $u\circ v^{-1}$ can be set as the inverse of an implicit canonical link function.

\section{Proof of Corollary \ref{corollary:uv_approximates_arbitrary_legendre_gradients}}
\label{appendix:uv_approximates_arbitrary_legendre_gradients}
Let $u = \text{softmax}^+$ and fix $g$ to be the gradient of a twice-differentiable Legendre function with positive Hessian everywhere. Note that $\text{LogSumExp}^+$ is twice-differentiable and Legendre so $u^{-1}$ and $g$ satisfy the sufficient conditions of Theorem \ref{theorem:maps_with_pd_jacobians_closed_under_composition}. It follows that $u^{-1} \circ g$ is the gradient of a twice-differentiable Legendre function defined on a compact set $\Omega$. The result then follows from using Proposition 3 of \citet{huang2021cpflows}.

\section{Experimental Details}
\label{appendix:experimental_details}

\subsection{Network Architecture and Optimisation Details}
Experiment details on architecture and optimisation parameters for \textsc{LegendreTron} (LT) and multinomial logistic regression (MLR). Here we denote $\alpha$ as the learning rate, $\lambda$ as weight decay, $\gamma$ as the multiplicative rate of decay applied to $\alpha$ every $S$ epochs through a step-wise learning rate scheduler. We used the Adam optimiser for all experiments.

\begin{tabular}{lccccccccc}
\toprule
Dataset(s) & Model & $B$ & $H$ & $M$ & $\alpha$ & $\gamma$ & $S$ & Epochs & Batch Size\\
\midrule
MNIST/FMNIST/KMNIST & LT & 1 & 4 & 4 & 0.001 & 0.7 & 4 & 200 & 128 \\
MNIST/FMNIST/KMNIST & MLR & $\backslash$ & $\backslash$ & $\backslash$ & 0.001 & 0.7 & 4 & 200 & 128 \\
aloi & LT & 2 & 2 & 4 & 0.01 & 0.95 & 4 & 360 & 64 \\
aloi & MLR & $\backslash$ & $\backslash$ & $\backslash$ & 0.01 & 0.95 & 4 & 360 & 64 \\
LIBSVM/UCI/Statlog (other datasets) & LT & 2 & 2 & 4 & 0.01 & 0.95 & 4 & 240 & 64 \\
LIBSVM/UCI/Statlog (other datasets) & MLR & $\backslash$ & $\backslash$ & $\backslash$ & 0.01 & 0.95 & 4 & 240 & 64 \\
\bottomrule
\end{tabular}

\subsection{LogSumExp trick for softmax$^+$}
Let $u=\text{softmax}^+$ and consider $\mathbf{x} \in \R^{C-1}$. We have 
\begin{align*}
    \log(\Pi^{-1}(u(\mathbf{x}))) = \left(\log\left(\frac{\exp(x_1)}{1+\sum_{k=1}^{C-1} \exp(x_k)}\right), \dots,\log\left(\frac{\exp(x_{C-1})}{1+\sum_{k=1}^{C-1} \exp(x_k)}\right),\log\left(\frac{1}{1+\sum_{k=1}^{C-1} \exp(x_k)}\right) \right)^{\top}
\end{align*}
where $\log$ on the LHS is applied elementwise. We seek an alternate expression for $\log(\Pi^{-1}(u(\mathbf{x})))$ that is numerically stable.

Let $x^* = \max(x_1, \dots,x_{C-1})$ and $S=\exp(-x^*)+\sum_{k=1}^{C-1} \exp(x_k-x^*)$. We can write 
\begin{align*}
    \log(\Pi^{-1}(u(\mathbf{x}))) &= \left(\log\left(\frac{\exp(x_1-x^*)}{S}\right), \dots,\log\left(\frac{\exp(x_{C-1}-x^*)}{S}\right),\log\left(\frac{\exp(-x^*)}{S}\right) \right)^{\top}\\
    &= \left(x_1-x^* - \log(S), \dots,x_{C-1}-x^* - \log(S),-x^* - \log(S) \right)^{\top}.
\end{align*}
It can be observed that this expression is numerically stable for all large values of $x^*$.


\end{document}